\documentclass[twoside,11pt]{article}

\usepackage{srcltx}
\usepackage{jmlr2e}
\usepackage{graphicx,subfigure}
\usepackage{epsfig} 
\usepackage{mathptmx} 
\usepackage{times} 
\usepackage{amsmath} 
\usepackage{amssymb} 
\usepackage{dsfont} \usepackage{verbatim}
\usepackage[svgnames]{xcolor}

\usepackage{pgfplots}\pgfplotsset{compat=1.4}

\usepackage{graphicx} 
\usepackage{hyperref}

\numberwithin{equation}{section}

\newcommand{\B}[1]{{\bf #1}} \newcommand{\Sc}[1]{{\mathcal{#1}}}
\newcommand{\R}[1]{{\rm #1}} \newcommand{\mB}[1]{{\mathbb{#1}}}

\newcommand{\set}[2]{\left\{#1\,\left\vert\, #2\right.\right\}}

\newcommand{\half}{\mbox{\small$\frac{1}{2}$}}
\DeclareMathOperator*{\argmin}{argmin}

\newcommand{\norm}[1]{\left\Vert #1\right\Vert}
\newcommand{\polar}[1]{#1^\circ} \newcommand{\horizon}[1]{#1^\infty}
\newcommand{\support}[2]{\delta^*\left(#1\left|\, #2\right.\right)}
\newcommand{\gauge}[2]{\gamma\left(#1\left|\, #2\right.\right)}
\newcommand{\indicator}[2]{\delta\left(#1\left|\, #2\right.\right)}
\newcommand{\ip}[2]{\left\langle #1,\, #2\right\rangle}
\newcommand{\dom}[1]{\R{dom}\left(#1\right)}
\newcommand{\lev}[2]{\R{lev}_{#1}\left(#2\right)}
\newcommand{\Snp}{\mathcal{S}^n_+}
\newcommand{\Snpp}{\mathcal{S}^n_{++}}
\newcommand{\ncone}[2]{N\left(#1\left|\, #2\right.\right)}
\newcommand{\proj}[3]{P_{#3}\left(#1\left|\, #2\right.\right)}
\newcommand{\dist}[2]{\mathrm{dist}\left(#1\left|\, #2\right.\right)}
\newcommand{\Rn}{\mB{R}^n} \newcommand{\Ran}[1]{\mathrm{Ran}\left(
    #1\right)}

\begin{document}

\title{Sprase/Robust Estimation and Kalman Smoothing with Nonsmooth Log-Concave Densities: Modeling, Computation, and Theory}
\footnote{The authors would like to thank Bradley M. Bell for insightful discussions and helpful suggestions.}

\author{
        \name Aleksandr Y. Aravkin \email saravkin@us.ibm.com\\
        \addr IBM T.J. Watson Research Center\\
         Yorktown, NY 10598\\
        \AND
        \name James V. Burke \email burke@math.washington.edu\\
        \addr Department of Mathematics, University of Washington\\
        Seattle, WA, USA
        \AND
        \name Gianluigi Pillonetto \email giapi@dei.unipd.it\\
        \addr Department of Information Engineering, University of Padova\\
        Padova, Italy}
\editor{}

\maketitle

\vspace{-.5in}
\begin{abstract}
\textcolor{black}{
We introduce a new class of quadratic support (QS) functions}, many of which
already play a crucial role in a variety of applications, 
including machine learning, robust statistical
inference, sparsity promotion, and inverse problems such as Kalman
smoothing. Well known examples of QS penalties include the $\ell_2$,
Huber, $\ell_1$ and Vapnik losses. 
\textcolor{black}{We build on a dual representation for QS functions, 
using it to characterize conditions necessary to interpret
these functions as negative logs of true probability densities.
This interpretation establishes the foundation for statistical modeling
 with both known and new QS loss functions, and 
enables construction of non-smooth multivariate distributions with specified means and
variances from simple scalar building blocks}. 

For a broad subclass of QS loss functions known as 
piecewise linear quadratic (PLQ) penalties, the dual 
representation allows for the development of efficient numerical estimation schemes.
\textcolor{black}{The main contribution of this paper is a 
flexible statistical modeling framework for a
variety of learning applications, together with 
a toolbox of efficient numerical methods for estimation using these densities.}
In particular, for PLQ densities, we show that interior point (IP) methods can be used. 
IP methods solve nonsmooth optimization problems 
by working directly with smooth systems of equations characterizing 
the optimality of these problems. We provide a few simple 
numerical examples, along with a code that can 
be used to prototype general PLQ problems. 

The efficiency of the IP approach depends on the structure
of particular applications. We consider the class of dynamic
inverse problems using Kalman smoothing. This class comprises a wide
variety of applications, where the aim is to reconstruct the state
of a dynamical system with known process and measurement models
starting from noisy output samples. In the classical case, Gaussian
errors are assumed both in the process and measurement models for
such problems. \textcolor{black}{We show that the extended framework allows arbitrary PLQ densities
to be used, and the that the proposed IP approach solves the
generalized Kalman smoothing problem while maintaining the linear
complexity in the size of the time series, just as in the Gaussian
case.} \textcolor{black}{This extends the computational efficiency of the Mayne-Fraser
and Rauch-Tung-Striebel algorithms to a much broader nonsmooth
setting, and includes many recently proposed robust and sparse
smoothers as special cases}. 
\end{abstract}


\section{Introduction}

Consider the classical problem of Bayesian parametric regression
\citep{MacKay,Roweis} where the unknown $x\in\mB{R}^n$ is a random
vector\footnote{All vectors are column vectors, unless otherwise specified},
with a prior distribution specified using a known invertible matrix $G\in\mB{R}^{n\times n}$ 
and known vector $\mu \in \mB{R}^n$ via
\begin{equation}
\label{LinearProcess}
\mu = Gx + w\;,
\end{equation}
where $w$ is a zero mean vector with covariance $Q$. 
Let $z$ denote a linear transformation of $x$
contaminated with additive zero mean measurement noise $v$ with covariance $R$,
\begin{equation}
\label{LinearModel}
z = H x + v\;,
\end{equation}
where $H\in\mB{R}^{\ell\times n}$ 
is a known matrix, while $v$ and $w$ are independent. 
It is well known that the (unconditional) minimum variance linear estimator of $x$,
as a function of $z$,
is the solution to the following optimization problem:
\begin{equation}
\label{MainObjective} 
\min_x
\quad (z - H x)^\R{T}R^{-1}(z - Hx)+ (\mu - G x)^\R{T}Q^{-1}(\mu - Gx)\;.
\end{equation}
As we will show, (\ref{MainObjective}) includes estimation problems
arising in discrete-time dynamic linear systems which admit a state
space representation \citep{Anderson:1979,Brockett}. In this
context, $x$ is partitioned into $N$ subvectors $\{x_k\}$, where
each $x_k$ represents the hidden system state at time instant $k$.
For known data $z$, the classical Kalman smoother exploits the
special structure of the matrices $H, G, Q$ and $R$ to compute the
solution of (\ref{MainObjective}) in $O(N)$ operations \citep{Gelb}.
This procedure returns the minimum variance estimate of the state
sequence $\{x_k\}$ when the
additive noise in the system is assumed to be Gaussian.\\
In many circumstances, the estimator (\ref{MainObjective})
performs poorly; put another way, quadratic penalization on
model deviation is a bad model in many situations.
For instance, it is not robust with
respect to the presence of outliers in the data
\citep{Hub,Gao2008,Aravkin2011tac,Farahmand2011} and may have
difficulties in reconstructing fast system dynamics, e.g. jumps in
the state values \citep{Ohlsson2011}. In addition,
sparsity-promoting regularization is often used in order to extract
a small subset from a large measurement or parameter vector which has
greatest impact on the predictive capability of the estimate for
future data. This sparsity principle permeates many well known
techniques in machine learning and signal processing, including
feature selection, selective shrinkage, and compressed sensing
\citep{Hastie90,LARS2004,Donoho2006}. In these cases,
(\ref{MainObjective}) is often replaced by a more general formulation
\begin{equation} 
\label{probTwo}
\min_{x} \quad  V \left(H x-z;R \right)  +  W\left(Gx-\mu; Q \right) \\
\end{equation}
where the loss $V$ may be the $\ell_2$-norm, the Huber penalty
\citep{Hub}, Vapnik's $\epsilon$-insensitive loss (used in support
vector regression \citep{Vapnik98} see also~\citep{Hastie01}) or the hinge loss
(leading to support vector classifiers
\citep{Evgeniou99,Pontil98,Scholkopf00}). The regularizer $W$ may be
the $\ell_2$-norm, the $\ell_1$-norm (as in the LASSO
\citep{Lasso1996}), or a weighted combination of the two, yielding
the elastic net procedure \citep{EN_2005}. Many  
learning algorithms using infinite-dimensional reproducing kernel
Hilbert spaces as hypothesis spaces \citep{Aronszajn,Saitoh,Cucker}
boil down to solving finite-dimensional problems of the form
(\ref{probTwo}) by
virtue of the representer theorem \citep{Wahba1998,Scholkopf01}.\\
These robust and sparse approaches can often be interpreted as
placing non-Gaussian priors on $w$ (or directly on $x$) and on the measurement noise $v$.
The Bayesian interpretation of (\ref{probTwo}) has been extensively
studied in the statistical and machine learning literature in recent
years and probabilistic approaches used in the analysis of
estimation and learning algorithms can be found e.g. in
\citep{McKayARD,Tipping2001,Wipf_IEEE_TIT_2011}. Non-Gaussian model
errors and priors leading to a great variety of loss and penalty
functions are also reviewed in \citep{Wipf_ARD_NIPS_2006} using
convex-type representations, and integral-type variational
representations related to
Gaussian scale mixtures. \\
In contrast to the above approaches, in the first part of the paper,
we consider a wide class of quadratic support (QS) functions
and exploit their dual representation. 
This class of functions generalizes the notion of  
piecewise linear quadratic (PLQ) penalties~\cite{RTRW}. 
The dual representation is the key to
identifying which QS loss functions can be associated with a density,
which in turn
allows us to interpret the solution to the problem (\ref{probTwo}) as
a MAP estimator when the loss functions $V$ and $W$ come from this
subclass of QS penalties.
This viewpoint allows statistical
modeling using non-smooth penalties, such as the
$\ell_1$, hinge, Huber and Vapnik losses, which are all PLQ penalties. 
\textcolor{black}{Identifying a statistical interpretation for this class of problems gives us several advantages, 
including a systematic constructive approach to prescribe 
mean and variance parameters for the corresponding model;
a property that is particularly important for Kalman smoothing.}
 \\
In addition,
the dual representation provides the foundation for efficient numerical methods 
in estimation based on interior point optimization technology.
In the second part of the paper, 
we derive the Karush-Kuhn-Tucker (KKT) equations 
for problem (\ref{probTwo}), and introduce interior point (IP) methods, 
which are iterative methods to solve the KKT equations using smooth approximations.  
This is essentially a smoothing approach 
to many (non-smooth) robust and sparse problems of interest to
practitioners.  Furthermore, we provide conditions under which the IP methods solve 
(\ref{probTwo}) when $V$ and $W$ come from PLQ densities, 
and describe implementation details for the entire class. \\
A concerted research effort has recently focused on the solution
of regularized large-scale inverse and learning problems, 
where computational costs and memory limitations are
critical. This class of problems includes the popular kernel-based methods
\citep{Rasmussen,Scholkopf01b,Smola:2003}, 
coordinate descent methods \citep{Tseng,Lucidi,Dinuzzo11} and
decomposition techniques \citep{Joachims,Lin,Lucidi}, one of which 
is the widely used sequential minimal optimization
algorithm for support vector machines \citep{Platt}. Other
techniques are based on kernel approximations, e.g. using incomplete
Cholesky factorization \citep{Fine}, approximate eigen-decomposition
\citep{Zhang} or truncated spectral representations
\citep{PilAuto2007}. Efficient interior point methods have been developed 
for $\ell_1$-regularized problems \citep{Kim}, and for support vector machines \citep{Ferris:2003}. \\
\textcolor{black}{In contrast, general and efficient solvers for state space estimation problems of the
form (\ref{probTwo}) are missing in the literature. 
The last part of this paper provides a contribution to fill this gap, 
specializing the general results to the dynamic case, and recovering
the classical efficiency results of the least-squares formulation. 
In particular, we design new Kalman smoothers tailored for
systems subject to noises coming from PLQ densities. Amazingly, it
turns out that the IP method used in \citep{Aravkin2011tac}
generalizes perfectly to the entire class of PLQ densities under a
simple verifiable non-degeneracy condition. }
In practice,  IP methods converge in a small number of iterations, 
and the effort per iteration depends on the structure of the underlying problem.
\textcolor{black}{We show that the IP iterations for all PLQ Kalman smoothing problems can be computed
with a number of operations that scales linearly in $N$, as in the
quadratic case. This theoretical foundation generalizes the results
recently obtained in
\citep{Aravkin2011tac,AravkinIFAC,Farahmand2011,Ohlsson2011}, framing
them as particular cases of the general framework presented here.}\\
The paper is organized as follows.
\textcolor{black}{In Section \ref{PLQP} we introduce the class of QS convex functions,
and give sufficient conditions that allow us to interpret
these functions as the negative logs of 
associated probability densities.
In Section \ref{PLQPTwo} we show how to
construct QS penalties and densities having a desired structure from basic components, 
and in particular how multivariate densities can be endowed with 
prescribed means and variances using scalar building blocks.
To illustrates this procedure, further details are provided for the 
Huber and Vapnik penalties.
In Section \ref{Optimization}, we focus on PLQ penalties, derive the associated KKT system,
and present a theorem that guarantees convergence of IP methods
under appropriate hypotheses.} 
 In Section~\ref{SimpleNumerics}, we present a few simple well-known problems, and compare a basic 
IP implementation for these problems with an ADMM implementation (all code is 
available online). 
In Section \ref{InteriorPointKS}, we present the Kalman smoothing dynamic model, 
\textcolor{black}
{formulate 
Kalman smoothing with PLQ penalties, present the KKT system for the dynamic case,
 and show that IP iterations for PLQ smoothing preserve the classical computational 
efficiency known for the Gaussian case. }
We present numerical examples using both simulated and real data 
in Section \ref{MethodComparisonSection}, and make some concluding remarks
in Section \ref{Conclusions}.
Section \ref{Appendix} serves as an appendix
where supporting mathematical results and proofs
are presented.

\section{Quadratic Support Functions and Densities}
\label{PLQP}

In this section, we introduce the class of Quadratic Support (QS) functions, 
characterize some of their properties, and show that many commonly
used penalties fall into this class. We also give a statistical interpretation 
to QS penalties by interpreting them as negative log likelihoods of probability densities;
this relationship allows prescribing means and variances along with the general
quality of the error model, an essential requirement of the Kalman smoothing
framework and many other areas. 

\subsection{Preliminaries}
\label{Prel}

We recall a few definitions from
convex analysis, required to specify the domains of 
QS penalties. The reader is referred to~\citep{RTR, RTRW} for 
more detailed reading.

\begin{itemize}
\item \index{affine hull} (Affine hull) Define the affine hull of any set
$C\subset\mB{R}^n$, denoted by $\R{aff}(C)$, as the smallest affine set 
(translated subspace) that
contains $C$.  
\item (Cone) For any set $C\subset\mB{R}^n$, denote by $\R{cone}\; C$ 
the set $\{tr | r \in C, t\in \mB{R}_+\}$.
\item (Domain) For  $f(x): \mathbb{R}^n\rightarrow \mathbb{\overline R} = \{\mB{R} \cup \infty\}$, 
$\R{dom}(f) = \{x: f(x) < \infty\}$.
\item \index{polar} (Polars of convex sets) For any convex set $C \subset \mB{R}^m$,
the polar of $C$ is defined to be
\[
\polar{C} := \{r | \langle r, d \rangle \leq 1 \; \forall \; d \in C\},
\]
and if $C$ is a convex cone, this representation is equivalent to 
\[
\polar{C} := \{r | \langle r, d \rangle \leq 0 \; \forall \; d \in C\}.
\]
\item (Horizon cone). Let $C\subset\mB{R}^n$ be a nonempty convex set. The horizon cone 
$\horizon{C}$ is the convex cone of `unbounded directions' for $C$,
i.e. $d\in \horizon{C}$ if $C + d \subset C$.
\item (Barrier cone). The barrier cone of a convex set $C$ is denoted by
  $\mathrm{bar}(C)$: 
\[
\mathrm{bar}(C) := \left\{x^* | \mbox{for some $\beta\in \mathbb{R}, \ip{x}{x^*}
\leq \beta\;\ \forall x \in C$}\right\}.
\]
\item (Support function). The support function for a set $C$ is denoted by 
$\support{x}{C}$:
\[
\support{x}{C} := \sup_{c\in C} \ip{x}{c}\;.
\]
\end{itemize}

\subsection{QS functions and densities}

We now introduce the QS functions and associated densities that are the focus of this paper.
We begin with the dual representation, which is crucial to both
establishing a statistical interpretation and to the development of a 
computational framework. 
\begin{definition}[Quadratic Support functions and penalties]
\label{generalPLQ}
A QS function is any function $\rho(U, M, b, B; \cdot): \mB{R}^n \rightarrow \mathbb{\overline R}$ 
having representation
\begin{equation}\label{PLQpenalty}
\rho(U, M, b, B; y) 
=
\sup_{u \in U}
\left\{ \langle u,b + By \rangle - \half\langle u, Mu
\rangle \right\} \;,
\end{equation}
where $U \subset \mB{R}^m$ is a nonempty convex set, 
$M\in \Snp$ the set of real symmetric positive semidefinite matrices,
and $b + By$ is an injective affine transformation in $y$, with $B\in\mB{R}^{m\times n}$, 
so, in particular, $m \leq n$ and $\R{null}(B) = \{0\}$. 

When $0 \in U$, we refer to the associated QS function as a {\it penalty}, since it 
is necessarily non-negative. 
\end{definition}
\begin{remark}
When $U$ is polyhedral, $0\in U$, $b = 0$ and $B = I$, we recover the basic piecewise linear-quadratic penalties 
characterized in \citep[Example 11.18]{RTRW}. 
\end{remark}
\begin{theorem}
\label{domainCharTheorem}  
Let $U, M, B, b$ be as in Definition~\ref{generalPLQ}, 
and set $K=\horizon{U} \cap \R{null}(M)$.
Then 
\[
B^{-1}[\mathrm{bar}(U)+\Ran{M} - b]\subset\R{dom}[\rho(U, M, B, b; \cdot)]\subset B^{-1}[\polar{K} - b]
\; ,
\]
with equality throughout when $\mathrm{bar}(U)+\Ran{M}$ is closed, where
$\mathrm{bar}(U)=\dom{\support{\cdot}{U}}$ is the barrier cone of $U$. In particular, equality
always holds when $U$ is polyhedral.
\end{theorem}

%
%
We now show that many commonly used penalties are special cases of  
QS (and indeed, of the PLQ) class.%

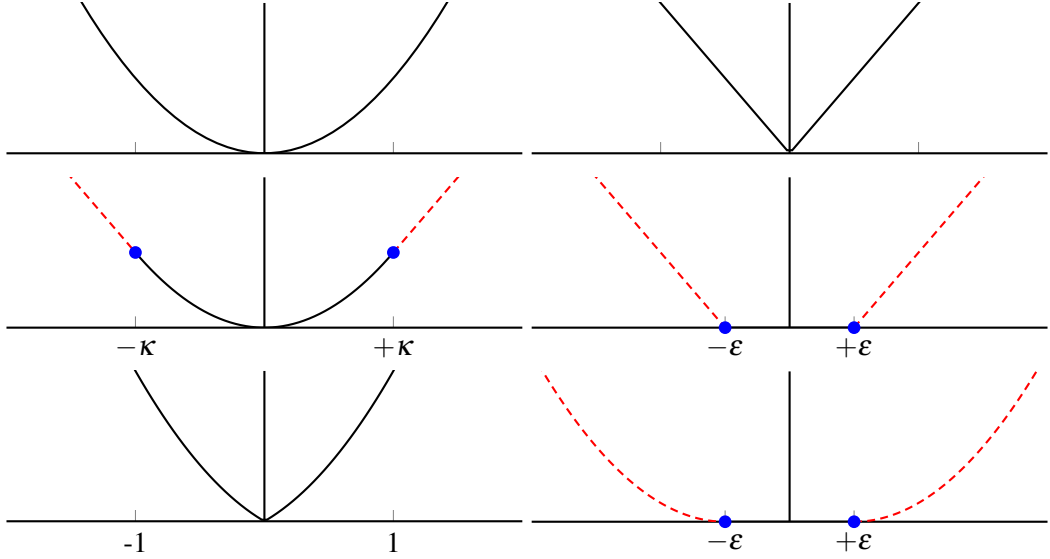
\begin{figure}
\centering
\begin{tikzpicture}
  \begin{axis}[
    thick,
    width=.45\textwidth, height=2cm,
    xmin=-2,xmax=2,ymin=0,ymax=1,
    no markers,
    samples=50,
    axis lines*=left, 
    axis lines*=middle, 
    scale only axis,
    xtick={-1,1},
    xticklabels={},
    ytick={0},
    ] 
\addplot[domain=-2:+2]{.5*x^2};
  \end{axis}
\end{tikzpicture}
\begin{tikzpicture}
  \begin{axis}[
    thick,
    width=.45\textwidth, height=2cm,
    xmin=-2,xmax=2,ymin=0,ymax=1,
    no markers,
    samples=100,
    axis lines*=left, 
    axis lines*=middle, 
    scale only axis,
    xtick={-1,1},
    xticklabels={},
    ytick={0},
    ] 
  \addplot[domain=-2:+2]{abs(x)};
  \end{axis}
\end{tikzpicture}
\begin{tikzpicture}
  \begin{axis}[
    thick,
    width=.45\textwidth, height=2cm,
    xmin=-2,xmax=2,ymin=0,ymax=1,
    no markers,
    samples=50,
    axis lines*=left, 
    axis lines*=middle, 
    scale only axis,
    xtick={-1,1},
    xticklabels={$-\kappa$,$+\kappa$},
    ytick={0},
    ] 
\addplot[red,domain=-2:-1,densely dashed]{-x-.5};
\addplot[domain=-1:+1]{.5*x^2};
\addplot[red,domain=+1:+2,densely dashed]{x-.5};
\addplot[blue,mark=*,only marks] coordinates {(-1,.5) (1,.5)};
  \end{axis}
\end{tikzpicture}
\begin{tikzpicture}
  \begin{axis}[
    thick,
    width=.45\textwidth, height=2cm,
    xmin=-2,xmax=2,ymin=0,ymax=1,
    no markers,
    samples=50,
    axis lines*=left, 
    axis lines*=middle, 
    scale only axis,
    xtick={-0.5,0.5},
    xticklabels={$-\epsilon$,$+\epsilon$},
    ytick={0},
    ] 
    \addplot[red,domain=-2:-0.5,densely dashed] {-x-0.5};
    \addplot[domain=-0.5:+0.5] {0};
    \addplot[red,domain=+0.5:+2,densely dashed] {x-0.5};
    \addplot[blue,mark=*,only marks] coordinates {(-0.5,0) (0.5,0)};
  \end{axis}
\end{tikzpicture}
\begin{tikzpicture}
  \begin{axis}[
    thick,
    width=.45\textwidth, height=2cm,
    xmin=-2,xmax=2,ymin=0,ymax=1,
    no markers,
    samples=100,
    axis lines*=left, 
    axis lines*=middle, 
    scale only axis,
    xtick={-1,1},
    xticklabels={-1, 1},
    ytick={0},
    ] 
\addplot[domain=-2:+2]{.5*x^2 + 0.5*abs(x)};
  \end{axis}
\end{tikzpicture}
\begin{tikzpicture}
  \begin{axis}[
    thick,
    width=.45\textwidth, height=2cm,
    xmin=-2,xmax=2,ymin=0,ymax=1,
    no markers,
    samples=50,
    axis lines*=left, 
    axis lines*=middle, 
    scale only axis,
    xtick={-0.5,0.5},
    xticklabels={$-\epsilon$,$+\epsilon$},
    ytick={0},
    ] 
    \addplot[red,domain=-2:-0.5,densely dashed] {0.5*(-x-0.5)^2};
    \addplot[domain=-0.5:+0.5] {0};
    \addplot[red,domain=+0.5:+2,densely dashed] {0.5*(x-0.5)^2};
    \addplot[blue,mark=*,only marks] coordinates {(-0.5,0) (0.5,0)};
  \end{axis}
\end{tikzpicture}
 \caption{ Scalar $\ell_2$ (top left), $\ell_1$ (top right), Huber (middle left), Vapnik (middle right), elastic net (bottom left) and smooth insensitive loss (bottom right) penalties}
\label{HuberVapnikFig}
\end{figure}

%
\begin{remark}[scalar examples]
\label{scalarExamples} 
$\ell_2$, $\ell_1$, elastic net, Huber, hinge, and Vapnik penalties 
are all representable using the notation of Definition
\ref{generalPLQ}.
\begin{enumerate}
\item\label{l2} $\ell_2$: Take $U = \mB{R}$, $M = 1$, $b = 0$, and $B = 1$. We obtain
\[ 
\displaystyle \rho(y) 
= \sup_{u \in \mB{R}}\left\{ uy -u^2/2 \right\}\;. 
\] 
The function inside the $\sup$ is
maximized at $u = y$, hence $\rho(y) = \frac{1}{2}y^2$, see top left 
panel of Fig.~\ref{HuberVapnikFig}. 

\item\label{l1} $\ell_1$: Take $U = [-1, 1]$, $M = 0$, $b = 0$, and $B = 1$. We obtain
\[ 
\displaystyle \rho(y) = \sup_{u \in [-1, 1]}\left\{ uy\right\}\;. 
\] The function inside the $\sup$ is maximized by
taking $u = \R{sign}(y)$, hence $\rho(y) = |y|$, see top right panel of Fig.~\ref{HuberVapnikFig}.

\item\label{elasticnet} Elastic net: $\ell_2 + \lambda \ell_1$. 
Take 
\[
U = \mB{R} \times [-\lambda, \lambda],\; 
b = \begin{bmatrix} 0 \\ 0 \end{bmatrix},\;
M = \begin{bmatrix} 1 & 0 \\ 0 & 0 \end{bmatrix},\;
B = \begin{bmatrix} 1 \\ 1\end{bmatrix}\;.
\]
This construction reveals the general calculus of PLQ addition, see Remark~\ref{sumsELQP}. See bottom right panel of Fig.~\ref{HuberVapnikFig}.

\item\label{huber} Huber: Take $U = [-\kappa, \kappa]$, $M = 1$, $b = 0$, and $B = 1$.
We obtain 
\[ \displaystyle \rho(y) = \sup_{u \in [-\kappa, \kappa]}\left\{ uy - u^2/2 \right\}\;,\] 
with three explicit cases:
\begin{enumerate}
\item If $ y < -\kappa $, take $u = -\kappa$ to obtain
$-\kappa y  -\frac{1}{2}\kappa^2$.
\item If $-\kappa  \leq y \leq \kappa $, take $u = y$ to obtain
$\frac{1}{2}y^2$.
\item If $y > \kappa  $, take $u = \kappa $ to obtain
a contribution of $\kappa y -\frac{1}{2}\kappa^2$.
\end{enumerate}
This is the Huber penalty, shown in the middle left
panel of Fig.~\ref{HuberVapnikFig}.

\item\label{hinge} Hinge loss:
Taking $B = 1$, $b = -\epsilon$, $M = 0$ and $U = [0, 1]$ we have 
\[
\rho(y) = \sup_{u \in U}\left\{ (y - \epsilon) u\right\} = (y - \epsilon)_+. 
\]
To verify this, just note that if $y < \epsilon$, $u^* = 0$; otherwise $u^* = 1$. 

\item\label{vapnik} Vapnik loss is given by $(y-\epsilon)_+  + ( -y- \epsilon)_+$. 
We immediately obtain its PLQ representation by  
taking 
\[
B = \begin{bmatrix}1 \\ -1\end{bmatrix},\;
b =  - \begin{bmatrix} \epsilon \\ \epsilon \end{bmatrix},\; 
M= \begin{bmatrix} 0 & 0 \\ 0 & 0\end{bmatrix}, \;
U = [0, 1]\times [0, 1]\;
\]
 to yield 
\[
\rho(y) 
= 
\sup_{u \in U}
\left\{
\left\langle \begin{bmatrix}y - \epsilon \\ 
-y - \epsilon \end{bmatrix} , u \right\rangle
\right\} 
= 
(y - \epsilon)_+ + (-y-\epsilon)_+. 
\]
The Vapnik penalty is shown in the middle
right panel of Fig.~\ref{HuberVapnikFig}.

\item\label{shlf} Soft hinge loss function ~\citep{Chu01aunified}.
Combining ideas from examples~\ref{huber} and~\ref{hinge}, we can construct a `soft' hinge loss;
i.e. the function 
\[
\rho(y)= \begin{cases}
0 & \text{if}\quad y < \epsilon\\
\frac{1}{2}(y - \epsilon)^2 & \text{if} \quad \epsilon < y < \epsilon + \kappa \\
\kappa (y - \epsilon) - \frac{1}{2}(\kappa)^2 &\text{if}\quad  \epsilon + \kappa < y\;.
\end{cases}
\]

that has a smooth (quadratic) transition rather than a kink at $\epsilon:$
Taking $B = 1$, $b = -\epsilon$, $M = 1$ and $U = [0, \kappa]$ we have 
\[
\rho(y) = \sup_{u \in [0,\kappa]}\left\{ (y - \epsilon) u\right\} -\half u^2\;.
\]
To verify this function has the explicit representation given above, note that if $y < \epsilon$, $u^* = 0$; if $\epsilon < y < \kappa + \epsilon$, we have $u^* = (y - \epsilon)_+$, 
and if $\kappa + \epsilon < y$, we have $u^*= \kappa$.

\item \label{silf} Soft insensitive loss function~\citep{Chu01aunified}. Using example~\ref{shlf}, 
we can create a symmetric soft insensitive loss function (which one might term the Hubnik) by 
adding together to soft hinge loss functions:  
\[
\begin{aligned}
\rho(y) &= \sup_{u \in [0,\kappa]}\left\{ (y - \epsilon) u\right\} -\half u^2 + \sup_{u \in [0,\kappa]}\left\{ (-y - \epsilon) u\right\} -\half u^2\\
& = \sup_{u \in [0,\kappa]^2}
\left\{
\left\langle \begin{bmatrix}y - \epsilon \\ 
-y - \epsilon \end{bmatrix} , u \right\rangle
\right\} 
- \half u^T \begin{bmatrix}1 & 0 \\ 0 & 1 \end{bmatrix}u\;.
\end{aligned}
\]
See bottom bottom right panel of Fig.~\ref{HuberVapnikFig}.
\end{enumerate}
\end{remark}
Note that the affine generalization (Definition \ref{generalPLQ}) is
needed to form the elastic net, the Vapnik penalty, and the SILF function, as all of these
are sums of simpler QS penalties. These sum constructions
are examples of a general calculus which allows the modeler to build up a QS
density having a desired structure.  This calculus is described in the following remark. 
\begin{remark}
\label{sumsELQP}
Let $\rho_1(y)$ and $\rho_2(y)$ be two QS penalties specified by
$U_i, M_i, b_i, B_i$, for $i = 1, 2$. Then the sum 
$\rho(y) = \rho_1(y) + \rho_2(y)$
is also a QS penalty, with 
\[
U = U_1 \times U_2,\; 
M = \begin{bmatrix} M_1 & 0 \\ 0 & M_2 \end{bmatrix},\;
b = \begin{bmatrix} b_1 \\ b_2\end{bmatrix},\;
B = \begin{bmatrix} B_1 \\ B_2 \end{bmatrix}\;. 
\]
\end{remark}
Notwithstanding the catalogue of scalar QS functions in Remark 4 and the gluing 
procedure described in Remark \ref{sumsELQP}, the supremum in Definition 
\ref{generalPLQ} appears to be a significant roadblock to understanding and designing a QS
function having specific properties. However, with some practice
the design of QS penalties is not as daunting a task as it first appears.
A key tool in understanding the structure of QS functions are Euclidean norm 
projections onto convex sets.

\begin{theorem}[Projection Theorem for Convex Sets]\label{projection theorem}[\cite{Z71}]
Let $Q\in\mB{R}^{n\times n}$ be symmetric and positive definite and let $C\subset\mB{R}$ be non-empty, closed
and convex. Then $Q$ defines an inner product on $\mB{R}^n$ by $\ip{x}{t}_Q=x^TQy$ with associated
Euclidean norm $\| x\|_Q=\sqrt{\ip{x}{x}_Q}$. The projection of a point $y\in\mB{R}^n$ onto $C$ in norm
$\|\cdot\|_Q$ is the unique point $P_Q(y\mid C)$ solving the least distance problem
\begin{equation}\label{proj dist}
\inf_{x\in C}\| y-x\|_Q,
\end{equation}
and 
$z=\proj{y}{C}{Q}$
if and only if $z\in C$ and
\begin{equation}\label{proj oc}
\ip{x-z}{y-z}_Q\le 0\quad\forall\; x\in C\ .
\end{equation}
\end{theorem}

Note that the least distance problem \eqref{proj dist} is equivalent to the problem
\[
\inf_{x\in C}\half \| y-x\|^2_Q\ .
\]
In the following lemma we use projections as well as duality theory  
to provide alternative representations for QS
penalties.

\begin{theorem}\label{QS representations}
Let $M\in\mB{R}^{n\times n}$ be symmetric and positive semi-definite matrix, let $L\in\mB{R}^{n\times k}$
be any matrix satisfying $M=LL^T$ where $k= \mathrm{rank}(M)$, and let
$U\subset\mB{R}^n$ be a non-empty, closed and convex set that contains the origin. Then the QS function
$\rho:=\rho(U,M,0,I;\cdot)$ has the primal representations
\begin{equation}\label{QS primal}
\rho(y)=\inf_{s\in\mB{R}^k}\left[ \half\Vert s\Vert^2_2+\support{y-Ls}{U}\right]
=\inf_{s\in\mB{R}^k}\left[ \half\Vert s\Vert^2_2+\gauge{y-Ls}{U^\circ}\right]\ ,
\end{equation}
where, for any convex set $V$,
\[
\support{z}{V}:=\sup_{v\in V}\ip{z}{v}\quad\mbox{ and }\quad 
\gauge{z}{V}:=\inf\set{t}{t\ge 0,\ z\in tV}
\]
are the support and gauge functionals for $V$, respectively.\\
If it is further assumed that $M\in\Snpp$ the set of positive definite matrices, 
then $\rho$ has the representations
\begin{eqnarray}
\rho(y)&=&
\inf_{s\in\mB{R}^k}\left[ \half\Vert s\Vert^2_{M}+\gauge{M^{-1}y-s}{M^{-1}U^\circ}\right]
\label{QS primal M}
\\ &=&
\half\| \proj{M^{-1}y}{U}{M}\|_{M}^2+\gauge{M^{-1}y-P_{M}(M^{-1}y|\, U)}{M^{-1}U^\circ}
\label{QS primal M 2}
\\ &=&
\inf_{s\in\mB{R}^k}\left[ \half\Vert s\Vert^2_{M^{-1}}+\gauge{y-s}{U^\circ}\right]
\label{QS primal M inv}
\\ &=&
\half\| P_{M^{-1}}(y|\, MU)\|_{M^{-1}}^2+\gauge{y-P_{M^{-1}}(y|\, MU)}{U^\circ}
\label{QS primal M inv 2}
\\
&=&\half y^TM^{-1}y-\inf_{u\in U}\half\Vert u-M^{-1}y\Vert_M^2
\label{QS U proj}
\\
&=&\half\|P_M(M^{-1}y|\, U)\|_M^2+\ip{M^{-1}y-P_M(M^{-1}y|\, U)}{P_M(M^{-1}y|\, U)}_{M}
\label{QS U proj 2}
\\
&=&\half y^TM^{-1}y-\inf_{v\in MU}\half\Vert v-y\Vert_{M^{-1}}^2
\label{QS MU proj}
\\
&=&\half\| P_{M^{-1}}(y|\, MU)\|_{M^{-1}}^2+\ip{y-P_{M^{-1}}(y|\, MU)}{P_{M^{-1}}(y|\, MU)}_{M^{-1}}\  .\label{QS MU proj 2}
\end{eqnarray}
In particular, \eqref{QS MU proj} says $\rho(y)=\half y^TM^{-1}y$
whenever $y\in MU$.
Also note that, by \eqref{QS primal}, one can replace the gauge functionals in 
\eqref{QS primal M}-\eqref{QS primal M inv 2} by the support functional of the appropriate set
where $M^{-1}U^\circ=(MU)^\circ$. 
\end{theorem}
The formulas \eqref{QS primal M}-\eqref{QS MU proj 2} show how one can build PLQ penalties 
having a wide range of desirable properties. 
%
We now give a short list of a few examples illustrating how to make use
of these representations.
\begin{remark}[General examples]
\label{generalExamples} 
In this remark we show how the representations in 
Lemma~\ref{QS representations} can be used to build 
QS penalties with specific structure. In each example we specify the components $U,M,b,$ and $B$
for the QS function
$\rho:=\rho(U,M,b,B;\cdot)$.
\begin{enumerate}
\item Norms. Any norm $\|\cdot\|$ 
can be represented as a QS function by taking $M = 0$, $B = I$, $b = 0$, $U = \polar{\mB{B}}$, 
where $\mB{B}$ is the unit ball of the desired norm. 
Then, by~\eqref{QS primal}, $\rho(y) = \|y\| = \gauge{y}{\mB{B}}$.

\item Gauges and support functions. 
Let $U$ be any closed convex set containing the origin, and Take $M = 0, B = I, b = 0$. 
Then, by \eqref{QS primal}, $\rho(y) = \gauge{y}{\polar{U}}=\support{y}{U}$.

\item Generalized Huber functions. Take any norm $\|\cdot\|$ having closed unit ball  $\mB{B}$. 
Let $M\in\Snpp$, $B = I$, $b = 0$, and $U = \polar{\mB{B}}$.
Then, by the representation~\eqref{QS primal M inv 2}, 
\begin{equation}\label{g-Huber}
\rho(y) = 
\half P_{M^{-1}}(y|\, M\polar{\mB{B}})^T{M^{-1}}P_{M^{-1}}(y|\, M\polar{\mB{B}})
+\| y-P_{M^{-1}}(y|\, M\polar{\mB{B}})\|\ .
\end{equation}
In particular, for $y\in M\polar{\mB{B}}$, $\rho(y)=\half y^TM^{-1}y$.

If we take $M=I$ and $\|\cdot\|=\kappa^{-1}\|\cdot\|_1$ for $\kappa >0$ 
(i.e. $U=\kappa \mB{B}_\infty$ and $U^\circ=\kappa^{-1}\mB{B}_1$), then $\rho$ is the 
multivariate Huber function described in item 4 of Remark \ref{scalarExamples}.
In this way, Theorem \ref{QS representations} shows how to generalize the essence of
the Huber norm to any choice of norm.
For example, if we take $U=\kappa\mB{B}_M=\set{\kappa u}{\|u\|_M\le 1}$,
then, by \eqref{QS U proj 2}, 
\[
\rho(y)=
\begin{cases}
\half\|y\|_{M^{-1}}^2&,\mbox{if }\|y\|_{M^{-1}}\le \kappa \\ 
\kappa \|y\|_{M^{-1}}-\frac{\kappa^2}{2}&,\mbox{if }\|y\|_{M^{-1}}> \kappa\ .
\end{cases}
\]

\item Generalized hinge-loss functions. Let $\|\cdot\|$ be a norm with closed unit ball $\mB{B}$,
let $K$ be a non-empty closed convex cone in $\Rn$, and let $v\in\Rn$. Set $M=0$, $b=-v$, $B=I$, and 
$U=-(\mB{B}^\circ\cap K^\circ)=\mB{B}^\circ\cap (-K)^\circ$. Then, by \cite[Section 2]{burke87tr},
\[
\rho(y)=\dist{y}{v-K}=\inf_{u\in K}\| y-b+u\|\ .
\]
If we consider the order structure ``$\le_K$'' induced on $\Rn$ by 
\[
y\le_Kv\quad\iff\quad v-y\in K\ ,
\]
then $\rho(y)=0$ if and only if $y\le_Kv$.
By taking $\|\cdot\|=\|\cdot\|_1$, $K=\Rn_+$ so $(-K)^\circ=K$, and $v=\epsilon\mathbf{1}$, 
where $\mathbf{1}$
is the vector of all ones, we recover the multivariate hinge loss function in Remark \ref{scalarExamples}.

\item Order intervals and Vapnik loss functions. 
Let $\|\cdot\|$ be a norm with closed unit ball $\mB{B}$,
let $K\subset\Rn$ be a non-empty symmetric convex cone in the sense that $K^\circ=-K$, 
and let $w<_Kv$, or equivalently, $v-w\in\mathrm{intr}(K)$. 
Set 
\[
U=(\mB{B}^\circ\cap K) \times (\mB{B}^\circ\cap K^\circ),\quad
M=\begin{bmatrix}0&0\\ 0&0\end{bmatrix},\quad 
b=-\begin{pmatrix}v\\ w\end{pmatrix},\quad\mbox{and}\quad
B=\begin{bmatrix}I\\ I\end{bmatrix}\ .
\]
Then
\[
\rho(y)=\dist{y}{v-K}+\dist{y}{w+K}.
\]
Observe that $\rho(y)=0$ if and only if $w\le_K y\le_K v$. 
The set $\set{y}{w\le_K y\le_K v}$
is an ``order interval'' \citep{Schaefer70}.
If we take $w=-v$, then
$\set{y}{-v\le_K y\le_K v}$ is a symmetric neighborhood of the origin.
By taking $\|\cdot\|=\|\cdot\|_1$, $K=\Rn_+$, and $v=\epsilon \mathbf{1}$=-w, 
we recover the multivariate Vapnik loss function in Remark \ref{scalarExamples}.
Further examples of symmetric cones are $\Snp$ and the Lorentz or $\ell_2$ cone
\citep{GH02}.
\end{enumerate}
\end{remark}
The examples given above show that one can also construct 
generalized versions of the elastic net as well
as the soft insensitive loss functions defined in Remark \ref{scalarExamples}. In addition, 
cone constraints can also be added by using the identity $\support{\cdot}{K^\circ}=\indicator{\cdot}{K}$.
These examples serve to illustrate the wide variety of penalty functions representable as QS functions.
Computationally, one is only limited by the ability to compute projections described in
Theorem \ref{QS representations}. Further computational properties for QS functions are
described in \citep[Section 6]{ABF12}.

In order to characterize QS functions as negative logs
of density functions, we need to ensure the integrability of said
density functions. The function $\rho(y)$ is said to be \emph{coercive} if
$\lim_{\|y\|\rightarrow \infty}\rho(y) = \infty$, and coercivity
turns out to be the key property to ensure integrability.
The proof of this fact and the characterization of
coercivity for QS functions are the subject of the
next two theorems (see Appendix for proofs).

\begin{theorem}[QS integrability]
\label{PLQIntegrability} Suppose $\rho(y)$ is
a coercive QS penalty. Then the function $\exp[-\rho(y)]$
is integrable on $\R{aff}[\R{dom}(\rho)]$ with respect to the
$\R{dim}(\R{aff}[\R{dom}(\rho)])$-dimensional Lebesgue measure.
\end{theorem}
%
%

\begin{theorem}
\label{coerciveRho} A QS function  $\rho$ is {\it coercive} if
and only if $\polar{[B^\R{T}\mathrm{cone}(U)]} = \{0\}$.
\end{theorem}

Theorem \ref{coerciveRho} can be used to show the coercivity
of familiar penalties. In particular, note that if $B=I$, then the QS
function is coercive if and only if $U$ contains the origin in its interior.
\begin{corollary} 
\label{coerciveCorollary}
The penalties $\ell_2$, $\ell_1$, elastic net, Vapnik, and
Huber are all coercive.
\end{corollary}
\begin{proof}
We show that all of these penalties satisfy the hypothesis of Theorem
\ref{coerciveRho}.
\begin{enumerate}
\item[]$\ell_2$: $U = \mB{R}$ and $B = 1$, 
so $\polar{\left[B^\R{T}\R{cone}(U)\right]} = \polar{\mB{R}} = \{0\}$.
\item[]$\ell_1$: $U = [-1, 1]$, 
so $\R{cone}(U) = \mB{R}$, and $B = 1$. 
\item[]Elastic Net: In this case, $\R{cone}(U) = \mB{R}^2$ and 
$B = \left[ \begin{smallmatrix} 1\\1 \end{smallmatrix} \right]$.
\item[]Huber: $U = [-\kappa,\kappa]$, so $\R{cone}(U) = \mB{R}$, 
and $B = 1$.
\item[]Vapnik: $U = [0,1] \times [0,1]$, so $\R{cone}(U) = \mB{R}^2_+$.
$B = \left[ \begin{smallmatrix} 1\\-1 \end{smallmatrix} \right]$, so
$B^\R{T}\R{cone}(U) = \mB{R}$. 
\end{enumerate}
\end{proof}
%
One can also show the coercivity of the above examples 
using their primal representations. 
However, our main objective is to pave the way for a modeling framework 
where multi-dimensional penalties can be constructed from simple 
building blocks and then solved by a uniform approach
using the dual representations alone. 

We now define a family of distributions on $\mB{R}^n$ by
interpreting piecewise linear quadratic functions $\rho$ as negative
logs of corresponding densities.
Note that the support of the distributions is always contained
in $\R{dom}\; \rho$, which is characterized in Theorem \ref{domainCharTheorem}.
\begin{definition}[QS densities]
\label{PLQDensityDef}
 Let $\rho(U, M, B, b; y)$ 
be any coercive extended QS penalty on 
$\mB{R}^n$. Define $\B{p}(y)$
to be the following density on $\mB{R}^n$:
\begin{equation}
\label{PLQdensity} \B{p}(y) =
\begin{cases}
c^{-1}\exp\left[- \rho(y)\right] & y \in \R{dom}\; \rho\\
0 & \R{else},
\end{cases}
\end{equation}
where
\[
c = \left(\int_{y \in \R{dom}\; \rho} \exp\left[-\rho(y)\right]dy\right),
\]
and the integral is with respect to the $\R{dim}(\R{dom}( \rho))$-dimensional Lebesgue measure.
\end{definition}

QS densities are true densities on the affine hull of the domain of $\rho$. 
The proof of Theorem \ref{PLQIntegrability} can be easily adapted to
show that they have moments of all orders.

\section{Constructing QS densities}
\label{PLQPTwo}

In this section, we describe how to construct multivariate QS densities 
with prescribed means and variances. We show how to compute normalization constants
to obtain scalar densities, and then extend to multivariate densities using linear transformations. 
Finally, we show how to obtain the data structures $U, M, B, b$ corresponding to multivariate
densities, since these are used by the optimization approach in Section~\ref{Optimization}.

We make use of the following definitions. 
Given a sequence of column vectors $\{ r_k \} = \{r_1, \dots, r_N\}$ 
and matrices $ \{ \Sigma_k\} = \{\Sigma_1, \dots, \Sigma_N\}$, we use the notation
\[
\R{vec} ( \{ r_k \} ) =
\begin{bmatrix}
r_1 \\ r_2  \\ \vdots \\ r_N
\end{bmatrix}
\; , \; \R{diag} ( \{ \Sigma_k \} ) =
\begin{bmatrix}
\Sigma_1    & 0      & \cdots & 0 \\
0      & \Sigma_2    & \ddots & \vdots \\
\vdots & \ddots & \ddots & 0 \\
0      & \cdots & 0      & \Sigma_N
\end{bmatrix} \; .
\]

In definition \ref{PLQDensityDef}, QS densities are defined
over $\mB{R}^n$. The moments of
these densities depend in a nontrivial
way on the choice of parameters $b, B, U, M$.
In practice, we would like to be able to construct these densities to
have prescribed means and variances. We show how this can be done
using  scalar QS random variables as the building blocks.
Suppose
${y} = \R{vec}(\{{y_k}\})$ is a vector of independent
(but not necessarily identical) QS random variables with mean $0$ and
variance $1$. Denote by 
$b_k, B_k, U_k, M_k$ the specification for the densities of ${y_k}$. To
obtain the density of ${y}$, we need only take
\[
\begin{aligned}
U &= U_1 \times U_2 \times \dots \times U_N\\
M &= \R{diag}(\{M_k\})\\
B &= \R{diag}(\{B_k\})\\
b &= \R{vec}(\{b_k\})\;.
\end{aligned}
\]
For example, the standard Gaussian distribution is specified by 
$U = \mB{R}^n$, $M = I$, $b = 0$, $B = I$, 
while the standard $\ell_1$-Laplace (see \citep{Aravkin2011tac}) 
is specified by $U = [-1,1]^n$, $M = 0$, $b=0$, $B = \sqrt{2}I$.\\
The random vector ${\tilde y} = Q^{1/2}({y}+\mu)$ 
has mean $\mu$ and variance $Q$.
If $c$ is the normalizing constant for the density
of ${y}$, then $c\det(Q)^{1/2}$ is
the normalizing constant for the density of ${\tilde y}$.\\
\begin{remark}
Note that only independence of the building blocks is required
in the above result. This allows the flexibility to impose different
QS densities on different errors in the model. Such flexibility
may be useful for example when
combining measurement data from different instruments, where some
instruments may occasionally give bad data (with outliers), 
while others have errors that are modeled well by Gaussian distributions.
\end{remark}
We now show how to construct
scalar building blocks with mean $0$ and variance $1$, i.e.
how to compute the key normalizing constants for any QS
penalty.
To this aim, suppose $\rho(y)$ is a scalar QS penalty that is symmetric
about $0$. We would like to construct a density
$\B{p}(y) = \exp\left[-\rho(c_2 y)\right]/c_1$
to be a true density with unit variance, that is,
\begin{equation}
\label{PLQconst}
\frac{1}{c_1} \int \exp\left[-\rho(c_2 y)\right]dy = 1\quad\mbox{and}\quad
\frac{1}{c_1}\int y^2\exp\left[-\rho(c_2 y)\right]dy = 1 ,
\end{equation}
where the integrals are over $\mB{R}$. Using $u$-substitution,
these equations become
\[
c_1c_2 
= \int \exp\left[-\rho(y)\right]dy\quad\mbox{and}\quad
c_1c_2^3 
= \int y^2\exp\left[-\rho(y)\right]dy .
\]

Solving this system yields 
\[
\begin{aligned}
c_2 
&= 
\sqrt{
	\left. \int y^2\exp\left[-\rho(y)\right]dy \right/ \int \exp\left[-\rho(y)\right]dy
}\\
c_1 &= \frac{1}{c_2} \int \exp\left[-\rho(y)\right]dy\;.
\end{aligned}
\]
These expressions can be used to obtain the normalizing constants
for any particular $\rho$ using simple integrals.

\subsection{Huber Density}

The scalar density corresponding to the
Huber penalty is constructed as follows. Set
\index{density!Huber scalar}
\begin{equation}
\label{HuberDensityScalar}
\B{p_H}(y) = \frac{1}{c_1}\exp[-\rho_H(c_2 y)] \;,
\end{equation}
where $c_1$ and $c_2$ are chosen as in (\ref{PLQconst}).
Specifically, we compute
\[
\begin{aligned}
\int \exp\left[-\rho_H(y)\right] dy
&=
2\exp\left[-\kappa^2/2\right]\frac{1}{\kappa} 
+ 
\sqrt{2\pi}[2\Phi(\kappa) - 1]\\
\int y^2\exp\left[-\rho_H(y)\right] dy
&=
4\exp\left[-\kappa^2/2\right]\frac{1+\kappa^2}{\kappa^3} 
+ 
\sqrt{2\pi}[2\Phi(\kappa) - 1]\;,
\end{aligned}
\]
where $\Phi$ is the standard normal cumulative density function.
The constants $c_1$ and $c_2$ can now be readily computed.\\
To obtain the multivariate Huber density with
variance $Q$ and mean $\mu$, let $U = [-\kappa, \kappa]^n$, $M = I$,
$B = I$ any full rank matrix, and $b = 0$.
This gives the desired density:
\index{density!Huber multivariate}
\begin{equation}
\label{HuberDensity}
\B{p_H}(y)
=
\frac{1}{c_1^n\det(Q^{1/2})}
\exp
\left[
-
\sup_{u \in U}\left\{
\left\langle
c_2Q^{-1/2}\left(y -\mu\right),
u \right\rangle
-
\frac{1}{2}u^\R{T}u
\right\}
\right].
\end{equation}

\subsection{Vapnik Density}

The scalar density associated with the Vapnik penalty
is constructed as follows. Set
\begin{equation}
\label{VapnikDensityScalar}
\B{p_V}(y) = \frac{1}{c_1}\exp\left[-\rho_V(c_2 y)\right] \;,
\end{equation}
where the normalizing constants $c_1$ and $c_2$ can be obtained
from
\index{density!Vapnik scalar}
\[
\begin{aligned}
\int \exp\left[-\rho_V(y)\right]dy &= 2(\epsilon + 1)\\
\int y^2\exp\left[-\rho_V(y)\right]dy &= \frac{2}{3}\epsilon^3 +
2(2 - 2\epsilon + \epsilon^2) ,
\end{aligned}
\]
using the results in Section \ref{PLQPTwo}. 
Taking $U = [0,1]^{2n}$, the multivariate Vapnik distribution 
with mean $\mu$ and variance $Q$ is
\index{density!Vapnik multivariate}
\begin{equation}
\label{VapnikDensity}
\B{p_V}(y)
=
\frac{1}{c_1^n\det(Q^{1/2})}
\exp
\left[
-
\sup_{u \in U}\left\{
\left\langle
c_2BQ^{-1/2}\left(y -\mu\right) - \epsilon\B{1}_{2n},
u \right\rangle
\right\}
\right]\;,
\end{equation}
where $B$ is block diagonal with each block of the form
 $B = \left[ \begin{smallmatrix} 1\\-1 \end{smallmatrix} \right]$, 
 and $\B{1}_{2n}$ is a column vector of $1$'s of length $2n$.

\section{Optimization with PLQ penalties}
\label{Optimization}

In the previous sections, QS penalties were  
characterized using their dual representation and 
interpreted as negative log likelihoods of true densities.
As we have seen, the scope of such densities is extremely broad. 
Moreover, these densities can
easily be constructed to possess specified moment properties.
In this section, we expand on their utility by
showing that the resulting estimation problems \eqref{probTwo} can be
solved with high accuracy
using standard techniques from numerical optimization for a large subclass of these penalties. 
We focus on PLQ penalties for the sake of simplicity in our presentation of an
interior point approach to solving these estimation problems.
However, the interior point approach applies in much more general settings, 
e.g. see \cite{NN94}. Nonetheless, the PLQ case is sufficient to cover all of the
examples given in Remark \ref{scalarExamples} while giving the flavor of how
to proceed in the more general cases.

%
We exploit the dual representation for the class of PLQ penalties \citep{RTRW} to 
explicitly construct the Karush-Kuhn-Tucker (KKT) conditions
for a wide variety of model problems of the form
(\ref{probTwo}). 
Working with these systems opens the door
to using a wide variety of numerical methods for convex
quadratic programming to solve~(\ref{probTwo}). 

Let $\rho(U_v, M_v, b_v, B_v; y)$ and $\rho(U_w, M_w, b_w, B_w; y)$ 
be two PLQ penalties and define
\begin{equation}\label{PLQ-V}
V(v;R):=\rho(U_v, M_v, b_v, B_v; R^{-1/2}v)
\end{equation}
and
\begin{equation}\label{PLQ-W}
W(w;Q):=\rho(U_w, M_w, b_w, B_w; Q^{-1/2}w).
\end{equation}
Then \eqref{probTwo} becomes
\begin{equation}\label{newopt}
\min_{y\in\mB{R}^n}\rho(U,M,b,B;y),
\end{equation}
where 
\[
U:=U_v\times U_w,\ M:=\begin{bmatrix}M_v&0\\ 0&M_w\end{bmatrix},\ 
b:=\begin{pmatrix}b_v-B_vR^{-1/2}z\\ b_w-B_wQ^{-1/2}\mu\end{pmatrix},
\]
and
\[ 
B:=\begin{bmatrix}B_vR^{-1/2}H\\ B_wQ^{-1/2}G\end{bmatrix}.
\]
Moreover, the hypotheses in \eqref{LinearProcess},
\eqref{LinearModel}, \eqref{probTwo}, and \eqref{PLQpenalty}
 imply that the matrix $B$ in \eqref{newopt}
is injective. Indeed, $By=0$ if and only if $B_wQ^{-1/2}Gy=0$, but, since $G$ is nonsingular
and $B_w$ is injective, this implies that $y=0$. That is, $\R{nul}(B)=\{0\}$.
Consequently, the objective in \eqref{newopt} takes the form of a PLQ penalty function 
\eqref{PLQpenalty}. In particular, if \eqref{PLQ-V} and \eqref{PLQ-W} arise from PLQ densities
(definition \ref{PLQDensityDef}), then the solution to problem \eqref{newopt} 
is the MAP estimator in the statistical model \eqref{LinearProcess}-\eqref{LinearModel}.

To simplify the notational burden, in the remainder of 
this section we work with \eqref{newopt} directly and
assume that the defining objects in \eqref{newopt} have the dimensions
specified in \eqref{PLQpenalty}; 
\begin{equation}\label{PLQdimSpecs}
U\in\mB{R}^m,\ M\in\mB{R}^{m\times m},\ 
b\in\mB{R}^m,\mbox{  and  }B\in \mB{R}^{m\times n}.
\end{equation}

The Lagrangian \citep{RTRW}[Example 11.47] for
problem (\ref{newopt}) is given by
\[
L(y, u) = b^\R{T}u -\frac{1}{2}u^\R{T}Mu + u^\R{T}By\;.
\]
By assumption $U$ is polyhedral, and so can be specified to take the form
\begin{equation}
\label{polyhedralU}
U = \{u: A^\R{T}u \leq a\}\;,
\end{equation}
where $A\in\mB{R}^{m\times \ell}$.
Using this reprsentation for $U$, the optimality
conditions for (\ref{newopt})~\citep{RTR,RTRW} are
\begin{equation}
\label{fullKKT}
\begin{aligned}
0 &= B^\R{T}u \\
0 &= b+By - Mu - Aq \\
0 &= A^\R{T}u + s - a\\
0 &= q_is_i \;, \  i=1,\dots,\ell\;, \; q, s \geq 0\;,
\end{aligned}
\end{equation}
where the non-negative slack variable $s$ is defined by the third equation
in \eqref{fullKKT}.
The non-negativity of $s$ implies that $u\in U$. 
The equations $0 = q_is_i \;, \  i=1,\dots,\ell$ in (\ref{fullKKT}) are known as the
complementarity conditions. 
By convexity, solving the problem (\ref{newopt}) is
equivalent to satisfying (\ref{fullKKT}). There is a vast
optimization literature on working directly with the KKT system. 
In particular, interior point (IP) methods~\citep{KMNY91,NN94,Wright:1997} can be employed.
In the Kalman filtering/smoothing application, IP methods
have been used to solve the KKT system (\ref{fullKKT}) in a
numerically stable and efficient manner, see e.g.
\citep{AravkinIFAC}. 
Remarkably, the IP approach used in \citep{AravkinIFAC}
generalizes to the entire PLQ class. For Kalman filtering and smoothing, the computational efficiency is also
preserved (see Section~\ref{InteriorPointKS}. 
Here, we show the general development for the entire PLQ class 
using standard techniques from the IP literature (see e.g.~\citep{KMNY91}).

Let $U, M, b, B,$ and $A$ be as defined in \eqref{PLQpenalty} and \eqref{polyhedralU},
and let $\tau\in(0,+\infty]$.
We define the {\it $\tau$ slice of the strict feasibility region for \eqref{fullKKT}} to be the set
\[
\mathcal{F}_+(\tau)=
\set{(s, q, u, y)}{\begin{array}{c}
0<s,\ 0< q,\ s^\R{T}q\le\tau,\mbox{ and }\\
(s, q, u, y)\mbox{ satisfy
the affine equations in \eqref{fullKKT}}\end{array}}\; ,
\]
and the {\it central path for \eqref{fullKKT}} to be the set
\[
\mathcal{C}:=\set{(s, q, u, y)}{\begin{array}{c}
0<s,\ 0<q,\ \gamma = q_is_i \; \  i=1,\dots,\ell,\mbox{ and }\\
(s, q, u, y)\mbox{ satisfy
the affine equations in \eqref{fullKKT}}\end{array}}\; .
\]
For simplicity, we define $\mathcal{F}_+:=\mathcal{F}_+(+\infty)$.
The basic strategy of a primal-dual IP method is to follow the central path
to a solution of \eqref{fullKKT} as $\gamma\downarrow 0$
by applying a predictor-corrector damped Newton 
method to the function mapping
$\mB{R}^\ell\times\mB{R}^\ell\times\mB{R}^m\times \mB{R}^n$
to itself
given by
\begin{equation}
\label{rKKT}
F_\gamma(s, q, u, y)
=
\begin{bmatrix}
s + A^\R{T}u - a \\
D(q)D(s)\B{1} - \gamma \B{1}\\
By - Mu - Aq +b\\
B^\R{T}u 
\end{bmatrix}\; ,
\end{equation}
where $D(q)$ and $D(s)$ are diagonal matrices 
with vectors $q, s$ on the diagonal. 
\begin{theorem}
\label{IPMtheorem}
Let $U, M, b, B,$ and $A$ be as defined in \eqref{PLQpenalty} and \eqref{polyhedralU}.
Given $\tau>0$, let $\mathcal{F}_+$, $\mathcal{F}_+(\tau)$, and $\mathcal{C}$ be as defined above.
If 
\begin{equation}\label{key IP}
\mathcal{F}_+\ne\emptyset\quad\mbox{and}\quad 
\R{null}(M)\cap \R{null}(A^\R{T}) =\{0\}, 
\end{equation}
then the following statements hold.
\begin{enumerate}
\item[(i)] 
$F^{(1)}_\gamma(s, q, u, y)$ is invertible for all $(s, q, u, y)\in\mathcal{F}_+$.
\item[(ii)] 
Define 
$\widehat{\mathcal{F}}_+=\set{(s,q)}{\exists\, (u,y)\in\mB{R}^m\times \mB{R}^n
\mbox{ s.t. }(s,q,u,y)\in \mathcal{F}_+}$. Then for each $(s,q)\in\widehat{\mathcal{F}}_+$
there exists a unique $(u,y)\in \mB{R}^m\times \mB{R}^n$ such that 
$(s,q,u,y)\in \mathcal{F}_+$.
\item[(iii)] 
The set $\mathcal{F}_+(\tau)$ is bounded
for every $\tau>0$.
\item[(iv)]
For every $g\in\mB{R}^\ell_{++}$, there is a unique 
$(s,q,u,y)\in \mathcal{F}_+$ such that $g=(s_1q_1,s_2q_2,\dots,s_\ell q_\ell)^\R{T}$.
\item[(v)] 
For every $\gamma>0$, there is a unique solution
$[s(\gamma),q(\gamma),u(\gamma),y(\gamma)]$ to
the equation $F_\gamma(s,q,u,y)=0$. Moreover, these points
form a differentiable trajectory in 
$\mB{R}^\nu\times\mB{R}^\nu\times\mB{R}^m\times \mB{R}^n$.
In particular, we may write
\[
 \mathcal{C}=\set{[s(\gamma),q(\gamma),u(\gamma),y(\gamma)]}{\gamma>0}\; .
\]
\item[(vi)]  
The set of cluster points of the central path as
$\gamma\downarrow 0$ is non-empty, and every such cluster point is a solution to
\eqref{fullKKT}.
\end{enumerate}
\end{theorem}
Please see the Appendix for proof. 
Theorem \ref{IPMtheorem} shows that if the conditions \eqref{key IP} hold, then IP
techniques can be applied to solve the problem \eqref{newopt}. 
In all of the applications we consider, the condition 
$\R{null}(M)\cap \R{null}(A^\R{T}) =\{0\}$ is easily verified. 
For example, in the setting of \eqref{newopt} with 
\begin{equation}
\label{UvUw}
U_v=\set{u}{A_vu\le a_v}\quad\mbox{and}\quad
U_w=\set{u}{A_wu\le b_w}
\end{equation}
 this condition reduces to
\begin{equation}
\label{NullProperty}
\R{null}(M_v)\cap \R{null}(A_v^\R{T}) =\{0\}
\quad\mbox{and}\quad
\R{null}(M_w)\cap \R{null}(A_w^\R{T}) =\{0\}.
\end{equation}
\begin{corollary}
\label{corExamples}
The densities corresponding to $\ell_1, \ell_2$, Huber,
and Vapnik penalties
all satisfy hypothesis~\eqref{NullProperty}.
\end{corollary}
\begin{proof}
We verify that $\R{null}(M) \cap \R{null}(A^\R{T}) = 0$
 for each of the four penalties.
 In the $\ell_2$ case, $M$ has full rank.
 For the $\ell_1$, Huber, and Vapnik penalties, the respective
sets $U$ are bounded, so $\horizon{U}= \{0\}$.
\end{proof}
On the other hand, 
the condition $\mathcal{F}_+\ne\emptyset$ is typically more difficult to verify.
We show how this is
 done for two sample cases from class \eqref{probTwo}, where the non-emptiness of
$\mathcal{F}_+$ is established by constructing an element of this set. 
Such constructed points are useful for initializing the interior point algorithm. 
 
 \subsection{$\ell_1$ -- $\ell_2$:}
 Suppose  $V(v;R)=\norm{R^{-1/2}v}_1$ and
 $W(w;Q)=\half\norm{Q^{-1/2}w}_2^2$. In this case
 \begin{gather*}
 U_v=[-\B{1}_m,\B{1}_m],\ M_v=0_{m\times m},\ b_v=0_m,\ B_v=I_{m\times m},\\
 U_w=\mB{R}^n,\ M_w=I_{n\times n},\ b_w=0_n,\ B_w=I_{n\times n},
 \end{gather*}
 and $R\in\mB{R}^{m\times m}$ and $Q\in\mB{R}^{n\times n}$ are
 symmetric positive definite covariance matrices. Following the notation of \eqref{newopt}
 we have
 \[
 U=[-\B{1},\B{1}]\times \mB{R}^n,\
 M=\begin{bmatrix}0_{m\times m}&0\\ 0&I_{n\times n}\end{bmatrix},\ 
 b=\begin{pmatrix}-R^{-1/2}z\\ -Q^{-1/2}\mu\end{pmatrix},\
 B=\begin{bmatrix}R^{-1/2}H\\ Q^{-1/2}G\end{bmatrix}.
 \]
 The specification of $U$ in \eqref{polyhedralU}  is given by
 \[
 A^\R{T}=\begin{bmatrix}I_{m\times m}&0_{n\times n}\\ -I_{m\times m}&0_{n\times n}\end{bmatrix}
 \mbox{ and }a=\begin{pmatrix}\B{1}\\  -\B{1}\end{pmatrix}\; .
 \]
 Clearly, the condition $\R{null}(M)\cap \R{null}(A^\R{T}) =\{0\}$ in \eqref{key IP} is 
 satisfied. Hence, for Theorem \ref{IPMtheorem} to apply, we need only check that
 $\mathcal{F}_+\ne\emptyset$. This is easily established by noting that 
 $(s,q,u,y)\in \mathcal{F}_+$, where
 \[
 u
 =
 \begin{pmatrix}
 0
 \\ 0
 \end{pmatrix},\
 y
 =
 G^{-1}\mu,\ 
 s=
 \begin{pmatrix}
 \B{1}\\ \B{1}
 \end{pmatrix},
 \ 
 q
 =
 \begin{pmatrix}
 \B{1}+[R^{-1/2}(Hy-z)]_+
 \\ \B{1}-[R^{-1/2}(Hy-z)]_-
 \end{pmatrix},
 \]
 where, for $g\in\mB{R}^\ell$, $g_+$ is defined componentwise by $g_{+(i)}=\max\{g_i,0\}$
 and $g_{-(i)}=\min\{g_i,0\}$.
 
 \subsection{Vapnik -- Huber:}
 Suppose that $V(v;R)$ and $W(w;Q)$ are as in \eqref{PLQ-V} and 
 \eqref{PLQ-W}, respectively, with $V$ a Vapnik penalty and $W$ a Huber penalty:
 \begin{gather*}
 U_v=[0,\B{1}_m]\times[0,\B{1}_m],\
 M_v=
 0_{2m\times 2m},\
 b_v=-\begin{pmatrix}\epsilon\B{1}_m\\ \epsilon\B{1}_m\end{pmatrix},\ 
 B_v=\begin{bmatrix}I_{m\times m}\\ -I_{m\times m}\end{bmatrix}\\
 U_w=[-\kappa\B{1}_n,\kappa\B{1}_n],\
 M_w=I_{n\times n},\
 b_w=0_n,\
 B_w=I_{n\times n}\ ,
 \end{gather*}
 and $R\in\mB{R}^{m\times m}$ and $Q\in\mB{R}^{n\times n}$ are
 symmetric positive definite covariance matrices. Following the notation of \eqref{newopt}
 we have
 \begin{gather*}
 U=([0,\B{1}_m]\times[0,\B{1}_m])\times [-\kappa\B{1}_n,\kappa\B{1}_n],\
 M=\begin{bmatrix} 0_{2m\times 2 m}&0\\ 0&I_{n\times n}\end{bmatrix},\\
 b=-\begin{pmatrix}\epsilon \B{1}_m +R^{-1/2}z\\ \epsilon \B{1}_m -R^{-1/2}z\\
      Q^{-1/2}\mu\end{pmatrix},\ 
B=\begin{bmatrix}R^{-1/2}H\\ - R^{-1/2}H\\ Q^{-1/2}G\end{bmatrix}\ .
 \end{gather*}
The specification of $U$ in \eqref{polyhedralU}  is given by
\[
A^\R{T}=\begin{bmatrix}I_{m\times m}&0&0\\ -I_{m\times m}&0&0\\
                                0&I_{m\times m}&0\\ 0&-I_{m\times m}&0\\
                                0&0&I_{n\times n}\\ 0&0&-I_{n\times n}\end{bmatrix}\mbox{ and }
a=\begin{pmatrix}\B{1}_m\\ 0 _m \\  \B{1}_m\\ 0 _m \\ \kappa\B{1}_n  \\  \kappa\B{1}_n
\end{pmatrix}\ .           
\]
Since $\R{null}(A^\R{T}) =\{0\}$,
the condition $\R{null}(M)\cap \R{null}(A^\R{T}) =\{0\}$ in \eqref{key IP} is 
 satisfied. Hence, for Theorem \ref{IPMtheorem} to apply, we need only check that
 $\mathcal{F}_+\ne\emptyset$. We establish this by constructing an element 
 $(s,q,u,y)$ of $\mathcal{F}_+$. For this, let 
 \[
 u=\begin{pmatrix}u_1\\ u_2\\ u_3\end{pmatrix},\ 
 s=\begin{pmatrix}s_1\\ s_2\\ s_3\\ s_4\\ s_5\\ s_6\end{pmatrix},\ 
 q=\begin{pmatrix}q_1\\ q_2\\ q_3\\ q_4\\ q_5\\ q_6\end{pmatrix},
 \]
 and set
 \[
y=0_n,\ u_1=u_2=\half\B{1}_\ell,\ u_3=0_n, \ s_1=s_2=s_3=s_4=\half\B{1}_\ell,\
s_5=s_6=\kappa\B{1}_n,
 \]
 and
 \begin{gather*}
 q_1=\B{1}_m -(\epsilon \B{1}_m+R^{-1/2}z)_-,\
 q_2=\B{1}_m +(\epsilon \B{1}_m+R^{-1/2}z)_+,\\
 q_3=\B{1}_m -(\epsilon \B{1}_m-R^{-1/2}z)_-,\
 q_4=\B{1}_m +(\epsilon \B{1}_m-R^{-1/2}z)_+,\\
q_5=\B{1}_n -(Q^{-1/2}\mu)_-,\
q_6=\B{1}_n +(Q^{-1/2}\mu)_+\ .
 \end{gather*}
 Then $(s,q,u,y)\in\mathcal{F}_+$. 

\section{Simple Numerical Examples and Comparisons}
\label{SimpleNumerics}

Before we proceed to the main application of interest (Kalman smoothing), 
we present a few simple and interesting problems in the PLQ class. 
An IP solver that handles the  problems discussed in this section is available through
\verb{github.com/saravkin/{, along with example files and ADMM implementations.
A comprehensive comparison with other methods 
is not in our scope, but we do compare the IP framework 
with the Alternating Direction Method of Multipliers (ADMM)(see~\cite{Boyd:2011} for a tutorial reference).
We hope that the examples and the code will help readers 
to develop intuition about these two methods. 
 
We focus on ADMM in particular because these 
methods enjoy widespread use in machine learning and other applications, due to
their versatility and ability to scale to large problems. The fundamental difference between ADMM and IP 
is that ADMM methods have at best linear convergence, so they cannot reach high accuracy 
in reasonable time (see~\cite[Section 3.2.2]{Boyd:2011}). 
In contrast, IP methods have a superlinear convergence rate 
(in fact, some variants have 2-step quadratic convergence, see~\cite{Ye:1993,Wright:1997}).

In addition to accuracy concerns, IP methods may be preferable to ADMM when
\begin{itemize}
\item objective contains complex non-smooth terms, e.g. $\|Ax -b\|_1$.
\item linear operators within the objective formulations are ill-conditioned. 
\end{itemize}

For formulations with well-conditioned linear operators and simple nonsmooth pieces (such as Lasso), 
ADMM can easily outperform IP.  
In these cases ADMM methods can attain moderate accuracy (and good solutions) very quickly, by exploiting 
partial smoothness and/or simplicity of regularizing functionals. 
For problems lacking these features, such as general formulations built from (nonsmooth) PLQ penalties 
and possibly ill-conditioned linear operators, 
IP can dominate ADMM, reaching the true solution while ADMM struggles.  

We present a few simple examples below, either developing the ADMM approach for each, or 
discussing the difficulties (when applicable). 
We explain advantages and disadvantages of using IP, and present numerical results. 
A simple IP solver that handles all of the examples, together with ADMM code used for the comparisons,
is available through \verb{github.com/saravkin/{. The Lasso example was taken directly from 
\verb{http://www.stanford.edu/~boyd/papers/admm/{, and we implemented the other ADMM examples
using this code as a template. 

\subsection{Lasso Problem}

Consider the Lasso problem 
\begin{equation}
\label{lasso}
\min_{x} \frac{1}{2}\|Ax - b\|_2^2 + \lambda \|x\|_1\; ,
\end{equation}
where $A\in\mathbb{R}^{n\times m}$. Assume that $m < n$. 
In order to develop an ADMM approach, we split the variables and introduce a constraint:
\begin{equation}
\label{splitLasso}
\min_{x,z} \frac{1}{2}\|Ax - b\|_2^2 + \lambda \|z\|_1 \quad \text{s.t.} \quad x = z\;.
\end{equation}
The augmented Lagrangian for~\eqref{splitLasso} is given by 
\begin{equation}
\label{LassoAugLag}
\Sc{L}(x, z, y) = \frac{1}{2}\|Ax - b\|_2^2 + \lambda \|z\|_1 + \eta y^T(z-x) + \frac{\rho}{2}\|z -x\|_2^2\;,
\end{equation}
where $\eta$ is the augmented Lagrangian parameter. The ADMM method now comprises the following
iterative updates:
\begin{equation}
\label{LassoADMM}
\begin{aligned}
x^{k+1} &= \argmin_x \frac{1}{2}\|Ax - b\|_2^2  + \frac{\eta}{2}\|x+ y^k-z^k \|_2^2\\
z^{k+1} &=  \argmin_z \lambda \|z\|_1 + \frac{\eta}{2}\|z -x^{k+1}+y^{k}\|_2^2 \\
y^{k+1} &=  y^k + (z^{k+1} - x^{k+1})\;.
\end{aligned}
\end{equation}

Turning our attention to the $x$-update, note that the gradient is given by 
\[
A^T(Ax - b) + \eta(x+ y^k-z^k) = (A^TA + I)x -A^Tb +\eta(y^k-z^k)\;.
\]
At every iteration, the update requires solving the same positive definite $m\times m$ symmetric system. 
Forming $A^TA + I$ is $O(nm^2)$ time, and obtaining a Cholesky factorization is $O(m^3)$, 
but once this is done, every $x$-update can be obtained in $O(m^2)$ time by doing 
two back-solves.

The $z$-update has a closed form solution given by soft thresholding:
\[
z^{k+1} = S(x^{k+1}-y^{k+1}, \lambda/\eta)\;,
\]
which is an $O(n)$ operation. The multiplier update is also $O(n)$. 
Therefore, the complexity per iteration is $O(m^2 + n)$, making ADMM a great method for this problem.

In contrast, {\it each iteration} of IP is dominated by the complexity of forming
a dense $m\times m$ system  $A^T D^k A$, where $D^k$ is a diagonal matrix that depends
on the iteration. So while 
both methods require an investment of $O(nm^2)$ to form and $O(m^3)$ to factorize the system, 
ADMM requires this only at the outset, while
IP has to repeat the computation for every iteration. 
A simple test shows ADMM can find a good answer, with a significant speed
advantage already evident for moderate ($1000\times 5000$) 
well-conditioned systems (see Table~\ref{table:SimpleResults}).

\subsection{Linear Support Vector Machines} 
The support vector machine problem can be formulated as the PLQ (see~\cite[Section 2.1]{Ferris:2003})
\begin{equation}
\label{SVM}
\min_{w, \gamma} \frac{1}{2} \|w\|^2 + \lambda \rho_+(1 - D(Aw- \gamma \B{1}))\;,
\end{equation} 
 where $\rho_+$ is the hinge loss function, $w^Tx = \gamma$ is the hyperplane being sought, 
 $D\in\mathrm{R}^{m\times m}$ is a diagonal matrix
 with $\{\pm 1\}$ on the diagonals (in accordance to the classification of the training data), 
 and $A \in \mathbb{R}^{m\times k}$ is the observation matrix, where each row gives the features corresponding 
 to observation $i \in \{1, \dots, m\}$. 
 The ADMM details are similar to the Lasso example, so we omit them here. The interested reader can 
 study the details in the file \verb{linear_svm{ available through \verb{github/saravkin{. 
 
 The SVM example turned out to be very interesting. We downloaded the 9th Adult example from the 
 SVM library at \verb{http://www.csie.ntu.edu.tw/~cjlin/libsvm/{. The training set has 32561 examples, 
 each with 123 features. When we formed the operator $A$ for problem~\eqref{SVM}, we found 
 it was very poorly conditioned, with condition number $7.7\times 10^{10}$. 
 It should not surprise the reader that after running for 653 iterations, ADMM is still 
 appreciably far away --- its objective value is higher, and in fact the relative norm distance to the (unique)
 true solution is 10\%. 
 
It is interesting to note that in this application, high optimization accuracy does not mean 
 better classification accuracy on the test set --- indeed, the (suboptimal) ADMM solution achieves a lower
 classification error on the test set (18\%, vs. 18.75\% error for IP). Nonetheless, this is not an advantage 
 of one method over another --- one can also stop the IP method early.  
The point here is that from the optimization perspective, SVM illustrates the advantages of Newton 
methods over methods with a linear rate.

\subsection{Robust Lasso}

For the examples in this section, we take $\rho(\cdot)$ to be 
a robust convex loss, either the 1-norm or the Huber function, 
and consider the robust Lasso problem
\begin{equation}
\label{RobustLasso}
\min_{x} \rho(Ax - b) + \lambda\|x\|_1\;.
\end{equation}

First, we develop an ADMM approach that works for both losses, exploiting the simple
nature of the regularizer. Then, we develop
a second ADMM approach when $\rho(x)$ is the Huber function by exploiting 
partial smoothness of the objective.  

Setting $z = Ax - b$, we obtain the augmented Lagrangian
\begin{equation}
\label{RobustLassoAugLagOne}
\Sc{L}(x, z, y) = \rho(z) + \lambda \|x\|_1 + \eta y^T(z-Ax+b) + \frac{\rho}{2}\|-z +Ax-b\|_2^2\;.
\end{equation}
The ADMM updates for this formulation are 
\begin{equation}
\label{RobustLassoADMMOne}
\begin{aligned}
x^{k+1} &= \argmin_x \lambda \|x\|_1 + \frac{\eta}{2}\|Ax-y^k-z^k \|_2^2\\
z^{k+1} &=  \argmin_z \rho(z)  + \frac{\eta}{2}\|z +y^k -Ax^{k+1}+b\|_2^2\\
y^{k+1} &=  y^k + (z^{k+1} - Ax^{k+1}+b)\;.
\end{aligned}
\end{equation}
The $z$-update can be solved using thresholding, or modified thresholding, 
in $O(m)$ time when $\rho(\cdot)$ is the Huber loss or 1-norm. 
Unfortunately, the $x$-update now requires solving a LASSO problem. 
This can be done with ADMM (see previous section), 
but the nested ADMM structure does not perform as well 
as IP methods, even for well conditioned problems. 

When $\rho(\cdot)$ is smooth, such as in the case of the Huber loss, 
 the partial smoothness of the objective can be exploited by setting $x =z$, 
obtaining \begin{equation}
\label{RobustLassoAugLagTwo}
\Sc{L}(x, z, y) = \rho(Ax - b) + \lambda \|z\|_1 + \eta y^T(zx) + \frac{\rho}{2}\|x-z\|_2^2\;.
\end{equation}
The ADMM updates are:
\begin{equation}
\label{RobustLassoADMMTwo}
\begin{aligned}
x^{k+1} &= \argmin_x \rho(Ax - b) + \frac{\eta}{2}\|x-z^k+y^k \|_2^2\\
z^{k+1} &=  \argmin_z \lambda \|z\|_1  + \frac{\eta}{2}\|z +(x^{k+1} + y^k)\|_2^2\\
y^{k+1} &=  y^k + (z^{k+1} - x^{k+1})\;.
\end{aligned}
\end{equation}
The problem required for the $x$-update is smooth, and can be solved by a
fast quasi-Newton method, such as L-BFGS. L-BFGS is implemented using 
only matrix-vector products, and for well-conditioned problems, the ADMM/LBFGS 
approach has a speed advantage over IP methods. 
For ill-conditioned problems, L-BFGS  has to work harder 
to achieve high accuracy, and inexact solves may destabilize the overall ADMM approach. 
IP methods are more consistent (see Table~\ref{table:SimpleResults}).

Just as in the Lasso problem, the IP implementation is dominated by the formation
of $A^TD^kA$ at every iteration with complexity $O(mn^2)$. However, a simple change
of penalty makes the problem much harder for ADMM, especially when the operator $A$
is ill-conditioned.

 \subsection{Complex objectives}
 
Many problems (including Kalman smoothers in the next section), do not have the simplifying
features exhibited by Lasso, SVM, and robust Lasso problems. Consider the general regression problem
\begin{equation}
\label{genProb}
\rho(Ax-b) + \|Cx\|_1\;,
\end{equation} 
where $\rho$ may be nonsmooth, and $C$ is in $\mathbb{R}^{k\times n}$. 

Applying ADMM to these objectives requires a bi-level implementation. 
For example, when $\rho(x)$ is the 1-norm, the 
$x$-update for ADMM requires solving 
\[
\min_x \|Ax - b\|_1 + \frac{\eta^2}{2}\|Bx - z - y\|_2^2\;,
\]
which is more computationally expensive than the Lasso subproblem. 
In particular, an ADMM implementation requires iteratively solving subproblems of the form 
\[
\min_x \|Bx - c\|_2^2 + \frac{\xi}{2}\|Ax - d\|_2^2\;.
\]
Since $B\in\mathbb{R}^{k\times n}$ and $A \in \mathbb{R}^{m\times n}$, 
a cholesky approach to the above problem requires forming an $n\times n$
matrix and factoring it. Since it was already observed that ADMM struggles to achieve moderate 
accuracy in the L1 Lasso case, we did not build an ADMM implementation in this more
general setting. 

However, applying the IP solver is straightforward, and  
we illustrate by solving the problem where $\rho(\cdot)$ is the 1-norm. In this case, 
the objective is a linear program with special structure, so it is  not
surprising that IP methods work well. 

We hope that the toy problems, results, and code that we developed in order to write 
this section have given the reader a better intuition for IP methods. 
Before moving on, note that the Kalman smoothing problems in the next section
have the flavor of the general L1-L1 example, since they must balance tradeoffs between
process and measurement models. Either penalty can be taken to be the 1-norm,
or any other PLQ penatly, and we will show that IP methods can be specifically designed to exploit 
the time series structure and preserve classical Kalman smoothing computational efficiency results.

\begin{table}
\label{table:SimpleResults}
\caption{For each problem, we give iteration counts for IP, outer ADMM iterations, 
the maximum cap for inner ADMM iterations (if applicable). 
We also give total computing time 
for both algorithms ($t_{ADMM}$, $t_{IP}$) on a 2.2 GHz dual-core Intel machine, 
and the objective difference f(xADMM) - f(xIP). This difference is always positive, 
since in all experiments IP found a lower objective value. Therefore, the magnitude of the objective
difference can be used as an accuracy heuristic for ADMM in each experiment, where lower difference
means higher ADMM accuracy. $\kappa(A)=$ condition number of $A$.
}
\begin{center}
\begin{tabular}{|c|c|c|c|c|c|c|}\hline
Problem & ADMM Iters& ADMM Inner & IP Iters & $t_{ADMM}$ (s) & $t_{IP}$ (s) & ObjDiff \\ \hline 
{\bf Lasso} &&&&&&
\\
$A: 1500\times 5000$ &  15& ---& 18 & 2.0 & 58.3 &  0.0025 \\ \hline
{\bf SVM} &&&&&&
\\
$\kappa(A) = 7.7\times 10^{10}$; $A: 32561\times 123$ &  653& ---& 77 & 41.2 & 23.9 &  0.17 \\ \hline
{\bf Huber Lasso} &&&&&&
\\
{\it ADMM/ADMM} &&&&&&
\\
$\kappa(A) = 5.8$; $A: 1000\times 2000$ &  26 & 100 & 20 & 14.1 & 10.5 &  $0.00006$ \\ 
$\kappa(A) = 1330$; $A: 1000\times 2000$ &  27 & 100& 24 & 40.0 & 13.0 &  $0.0018$ \\
{\it ADMM/L-BFGS} &&&&&&
\\
$\kappa(A) = 5.8$; $A: 1000\times 2000$ &  18&--- & 20 & 2.8 & 10.3 &  $1.02$ \\ 
$\kappa(A) = 1330$; $A: 1000\times 2000$ &  22&--- & 24 & 21.2 & 13.1 &  $1.24$ \\ \hline
{\bf L1 Lasso} &&&&&& 
\\
ADMM/ADMM&&&&&&
\\
$\kappa(A) = 2.2$; $A: 500\times 2000$ &  104 & 100 &29 & 57.4 & 5.9 &  $0.06$ \\ 
$\kappa(A) = 1416$; $A: 500\times 2000$ &  112 & 100& 29 & 81.4 & 5.6&0.21 \\ \hline

{\bf General L1-L1} &&&&&& 
\\
$C: 500\times 2000$;
$A: 1000\times 2000$ &  --- & --- & 11 & --- &21.4&--- \\\hline 
\end{tabular}
\end{center}
\end{table}

\section{Kalman Smoothing with PLQ penalties}
\label{InteriorPointKS}

Consider now a dynamic scenario, where the system state $x_k$
evolves according to the following stochastic discrete-time linear model
\begin{equation}
\label{LinearGaussModel}
\begin{aligned}
    x_1&=x_0 + w_1
   \\
    x_k & =  G_k x_{k-1}  + w_k,   \qquad k=2,3,\ldots,N
    \\
    z_k & =  H_k x_k      + v_k,    \quad \qquad k=1,2,\ldots,N
\end{aligned}
\end{equation}
where $x_0$ is known, $z_k$ is the $m$-dimensional subvector of $z$
containing the noisy output samples collected at instant $k$, $G_k$
and $H_k$ are known matrices. Further, we consider
the general case where $\{w_k\}$ and $\{v_k\}$ are
mutually independent zero-mean random variables which
can come from any of the
densities introduced in the previous section, with positive
definite covariance matrices denoted by $\{Q_k\}$ and $\{R_k\}$, respectively.\\
In order to formulate the Kalman smoothing problem over the entire
sequence $\{x_k\}$, define
\[
\begin{aligned}
& x = \R{vec}\{x_1, \cdots, x_N\}\;,\qquad
&w = \R{vec}\{w_1, \cdots, w_N\}\; \\
&v = \R{vec}\{v_1, \cdots, v_N\}\;, \qquad
&Q = \R{diag}\{Q_1, \cdots, Q_N\}\;\\
&R = \R{diag}\{R_1, \cdots, R_N\}\;, \qquad
&H = \R{diag}\{H_1, \cdots, H_N\},
\end{aligned}
\]
%
and
\[
G =
\begin{bmatrix}
    \R{I}  & 0      &          &
    \\
    -G_2   & \R{I}  & \ddots   &
    \\
        & \ddots &  \ddots  & 0
    \\
        &        &   -G_N   & \R{I}
\end{bmatrix}\;
\]
Then model (\ref{LinearGaussModel}) can be written in the form of
\eqref{LinearProcess}-\eqref{LinearModel}, i.e.,
\begin{equation}
\label{fullStat}
\begin{aligned}
\mu
&=
Gx + w\\
z
&=
Hx + v\;,
\end{aligned}
\end{equation}
where $x \in \mB{R}^{nN}$ is the entire state sequence of
interest, $w$ is corresponding process noise,
$z$ is the vector of all measurements,
$v$ is the measurement noise, and
$\mu$ is a vector of size $nN$ with the first
$n$-block equal to $x_0$, the initial state estimate,
and the other blocks set to $0$. 
This is precisely the problem \eqref{LinearProcess}-\eqref{LinearModel} that
began our study.
The problem (\ref{MainObjective})
becomes the classical Kalman smoothing
problem with quadratic penalties.
In this case, the objective function can be written 
\[
\|Gx - \mu\|_{Q^{-1}}^2 + \|Hx - z\|_{R^{-1}}^2,
\]
 and the minimizer can be found by taking the gradient and setting it to zero: 
\[
(G^TQ^{-1}{G} + H^TR^{-1}H)x = r\;.
\]
One can view this as a single step of Newton's method, which converges to the solution because the objective is quadratic. 
Note also that once the linear system above is formed, it takes only $O(n^3N)$ operations to solve due to special block 
tridiagonal structure (for a generic system, it would take $O(n^3N^3)$ time).
In this section, we will show that IP methods can preserve this complexity for much more general penalties 
on the measurement and process residuals. We first make a brief remark 
related to the statistical interpretation of PLQ penalties. 

\begin{remark}
Suppose we decide  to move to an outlier robust formulation, where the 1-norm or Huber penalties are used, but  
the measurement variance is known to be $R$. Using the statistical interpretation developed in section~\ref{PLQPTwo},
the statistically correct objective function for the smoother is 
\[
\small
\frac{1}{2}\|Gx - \mu\|_{Q^{-1}}^2 + \sqrt{2}\|R^{-1}(Hx - z)\|_1\;.
\]
Analogously, the statistically correct objective when measurement error is the Huber penalty with parameter $\kappa$ is  
\[
\frac{1}{2}\|Gx - \mu\|_{Q^{-1}}^2 + c_2\rho(R^{-1/2}(Hx - z))\;,
\]
where 
\[
\small
c_2 = \frac{4\exp\left[-\kappa^2/2\right]\frac{1+\kappa^2}{\kappa^3} 
+ 
\sqrt{2\pi}[2\Phi(\kappa) - 1]
}
{2\exp\left[-\kappa^2/2\right]\frac{1}{\kappa} 
+ 
\sqrt{2\pi}[2\Phi(\kappa) - 1]}.
\]
The normalization constant comes from the results in Section 3.1, and ensures that the 
weighting between process and measurement  terms is still consistent with the situation 
{\it regardless of which shapes are used for the process and measurement penalties}. 
This is one application of the statistical interpretation. 
\end{remark}

Next, we show that when the penalties used on the process residual $Gx-w$ and measurement residual $Hx-z$
arise from general PLQ densities, the general Kalman smoothing problem takes the form \eqref{newopt},
studied in the previous section. The details are given in the following remark. 

\begin{remark}
\label{propKS}
Suppose that the noises $w$ and $v$ in the model~\eqref{fullStat}
are PLQ densities with means $0$, variances  $Q$ and $R$
(see Def. \ref{PLQDensityDef}). Then, for suitable
$U_w, M_w,b_w,B_w$ and $U_v, M_v,b_v,B_v$ and corresponding 
$\rho_w$ and $\rho_v$
we have 
\begin{equation}
\label{kalmanDensities}
\begin{aligned}
\B{p}(w)&\propto \exp\left[-\rho\left(U_w, M_w, b_w, B_w; Q^{-1/2}w\right)\right] \\
 \B{p}(v) &\propto \exp\left[-\rho(U_v, M_v, b_v, B_v; R^{-1/2}v)\right]\;
\end{aligned}
\end{equation}
while the MAP estimator
of $x$ in the model~\eqref{fullStat} is
\begin{equation}
\argmin_{x\in \mB{R}^{nN}}
\left\{
\begin{aligned}
\label{PLQsubproblem}
&\rho\left[U_w,M_w, b_w, B_w; Q^{-1/2}(Gx - \mu)\right] \\
&+\rho\left[U_v, M_v, b_v, B_v; R^{-1/2}(Hx - z)\right]
\end{aligned}
\right\}\;
\end{equation}
\end{remark}
If $U_w$ and $U_v$ are given as in~\eqref{UvUw}, 
then the system~\eqref{fullKKT} decomposes as 
%
%
%
\begin{equation}
\label{PLQFinalConditions}
\begin{array}{lll}
&\begin{array}{llllll}
0 &=& A_w^\R{T}u_w + s_w - a_w\;;&&  0=  A_v^\R{T}u_v + s_v - a_v\\
0 &=& s_w^\R{T}q_w\;; && 0= s_v^\R{T}q_v
\end{array}\\
&\begin{array}{llllll}
0 &=& \tilde b_w + B_w Q^{-1/2}G{d} -  M_w{u}_w - A_wq_w\\
0 &=& \tilde b_v - B_v R^{-1/2}H{d} - M_v{u}_v -  A_v q_v\\
0 &=& G^\R{T}Q^{-\R{T}/2}B_w^\R{T} u_w -
H^\R{T}R^{-\R{T}/2}B_v^\R{T} u_v\\
0 &\leq& s_w, s_v, q_w, q_v.
\end{array}
\end{array}
\end{equation}
%
See the Appendix and~\citep{AravkinThesis2010} for details on deriving the KKT system. 
By further exploiting the decomposition shown in~\eqref{LinearGaussModel},
we obtain the following theorem. 

\begin{theorem}[PLQ Kalman smoother theorem]
\label{thmPLQsmoother}
Suppose that all $w_k$ and $v_k$ in
the Kalman smoothing model (\ref{LinearGaussModel}) come from PLQ
densities that satisfy
\begin{equation}
\label{NullPropertyKS}
\R{null}(M_k^w)\cap \R{null}((A_k^w)^\R{T})  
= \{0\}\;,  
\R{null}(M_k^v)\cap \R{null}((A_k^v)^\R{T})  
=
\{0\}\;,\;\forall k\;,
\end{equation} i.e.
their corresponding penalties are finite-valued.
Suppose further that the corresponding set $\mathcal{F}_+$ 
from Theorem~\ref{IPMtheorem} is nonempty. 
Then~\eqref{PLQsubproblem} can be solved using an IP method,
with computational complexity $O[N(n^3 + m^3+l)]$, 
where $l$ is the largest 
column dimension of the matrices $\{A_k^\nu \}$ and $\{A_k^w \}$. 
\end{theorem}
Note that the first part of this theorem, 
the solvability of the problem using IP methods, 
already follows from Theorem \ref{IPMtheorem}. 
The main contribution of the result in the dynamical 
system context is the computational complexity. 
The proof is presented in the Appendix
and shows that IP methods for
solving~\eqref{PLQsubproblem}
preserve the key block tridiagonal structure of the
standard smoother. If the number of IP iterations is fixed
($10-20$ are typically used in practice), general smoothing
estimates can thus be computed in $O[N(n^3+ m^3 + l)]$ time.
Notice also that the number of required operations scales
linearly with $l$, which represents the 
complexity of the PLQ density encoding.\\
%

\section{Numerical example}
\label{MethodComparisonSection}

\subsection{Simulated data}

\begin{figure*}
\begin{center}
{\includegraphics[scale=0.6]{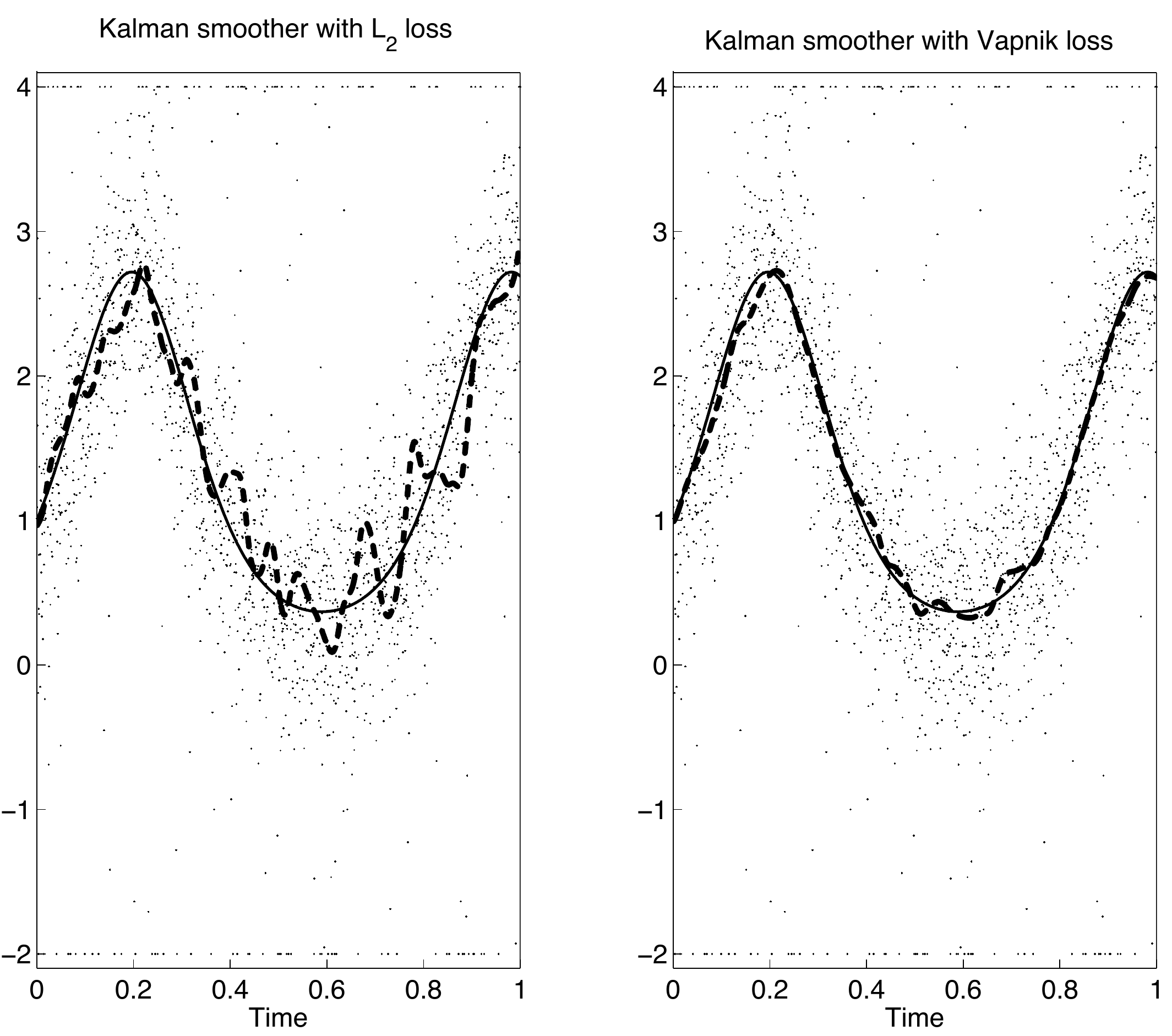}}
\end{center}
\caption{
    \label{FigLtwoVap}
    Simulation:
    measurements ($\cdot$) with outliers
    plotted on axis limits ($4$ and $-2$),
    true function (continuous line),
    smoothed estimate using either the quadratic loss (dashed line, left panel)
    or the Vapnik's $\epsilon$-insensitive loss (dashed line, right panel)
}
\end{figure*}

In this section we use a simulated example to test the
computational scheme described in the previous section.
We consider the following function
\begin{equation*}
f(t)=\exp\left[\sin(8t)\right]
\end{equation*}
taken from \citep{Dinuzzo07}. Our aim is to reconstruct $f$
starting from
2000 noisy samples collected uniformly over the unit interval.
The measurement noise $v_k$ was generated using a mixture of two Gaussian densities, with
$p=0.1$ denoting the fraction from each Gaussian; i.e.,
\[
v_k \sim (1 - p ) \B{N} (0, 0.25 ) + p \B{N} ( 0 , 25 ),
\]
Data are displayed as dots in Fig.~\ref{FigLtwoVap}.
Note that the purpose of the second component of the Gaussian mixture
is to simulate outliers in the output data and that all the measurements exceeding
vertical axis limits are plotted on upper and lower axis limits (4 and -2) to improve readability.\\
The initial condition $f(0) = 1$ is assumed to be known, 
while the difference of the unknown function from the initial condition 
(i.e. $f(\cdot) - 1$) is modeled as a Gaussian process
given by an integrated Wiener process. This model captures the Bayesian
interpretation of cubic smoothing splines \citep{Wahba1990}, 
and admits a two-dimensional state space representation
where the first component of $x(t)$, which models $f(\cdot) - 1$,
corresponds to the integral of the second state component, modelled as Brownian motion.
To be more specific, letting $\Delta t = 1/2000$, the sampled version
of the state space model (see \citep{Jaz,Oks} for details) is defined by
\begin{eqnarray*}
G_k =
\begin{bmatrix}
    1        & 0
    \\
    \Delta t & 1
\end{bmatrix}, \qquad k=2,3,\ldots,2000\\
H_k =
\begin{bmatrix}
    0        & 1
\end{bmatrix}, \qquad k=1,2,\ldots,2000
\end{eqnarray*}
with the autocovariance of $w_k$ given by
\[
Q_k = \lambda^2
\begin{bmatrix}
    \Delta t   & \frac{\Delta t^2}{2}
    \\
    \frac{\Delta t^2}{2} & \frac{\Delta t^3}{3}
\end{bmatrix}, \qquad k=1,2,\ldots,2000
\; ,
\]
where $\lambda^2$ is an unknown scale factor
to be estimated from the data.\\
We compare the performance of two Kalman smoothers.
The first (classical) estimator uses a quadratic loss function
to describe the negative log of the measurement noise density and
contains only $\lambda^2$ as unknown parameter.
The second estimator is a Vapnik smoother relying
on the $\epsilon$-insensitive loss,
and so depends on two unknown parameters: $\lambda^2$ and $\epsilon$.
In both cases, the unknown parameters are estimated
by means of a cross validation strategy where the 2000 measurements
are randomly split into a training and a validation set of 1300
and 700 data points, respectively. The Vapnik smoother
was implemented by exploiting the efficient computational
strategy described in the previous section, see~\citep{AravkinIFAC} for 
specific implementation details.
Efficiency is particularly important here, because of the need for cross-validation. 
In this way, for each value of $\lambda^2$ and $\epsilon$
contained in a $10 \times 20$ grid on $[0.01,10000] \times [0,1]$,
with $\lambda^2$ logarithmically spaced,
the function estimate was rapidly obtained by the new smoother applied to the training set.
Then,  the relative average prediction error on the validation set was computed,
see Fig.~\ref{FigVal}.
The parameters leading to the best prediction were 
$\lambda^2=2.15\times 10^3$ and $\epsilon=0.45$, 
which give a sparse solution
defined by fewer than 400 support vectors. 
The value of $\lambda^2$
for the classical Kalman smoother was then estimated following the
same strategy described above. 
In contrast to the Vapnik penalty,
the quadratic loss does not induce any sparsity, so that, 
in this case, the number of support vectors
equals the size of the training set.\\
The left and right panels of Fig.~\ref{FigLtwoVap} display the function estimate
obtained using the quadratic and the Vapnik losses, respectively.
It is clear that the estimate obtained using the quadratic penalty is heavily affected by the outliers. 
In contrast, as expected, the estimate coming from the Vapnik based smoother 
performs well over the entire time period, 
and is virtually unaffected by the presence of large outliers.

\begin{figure}
\begin{center}
{\includegraphics[scale=0.3]{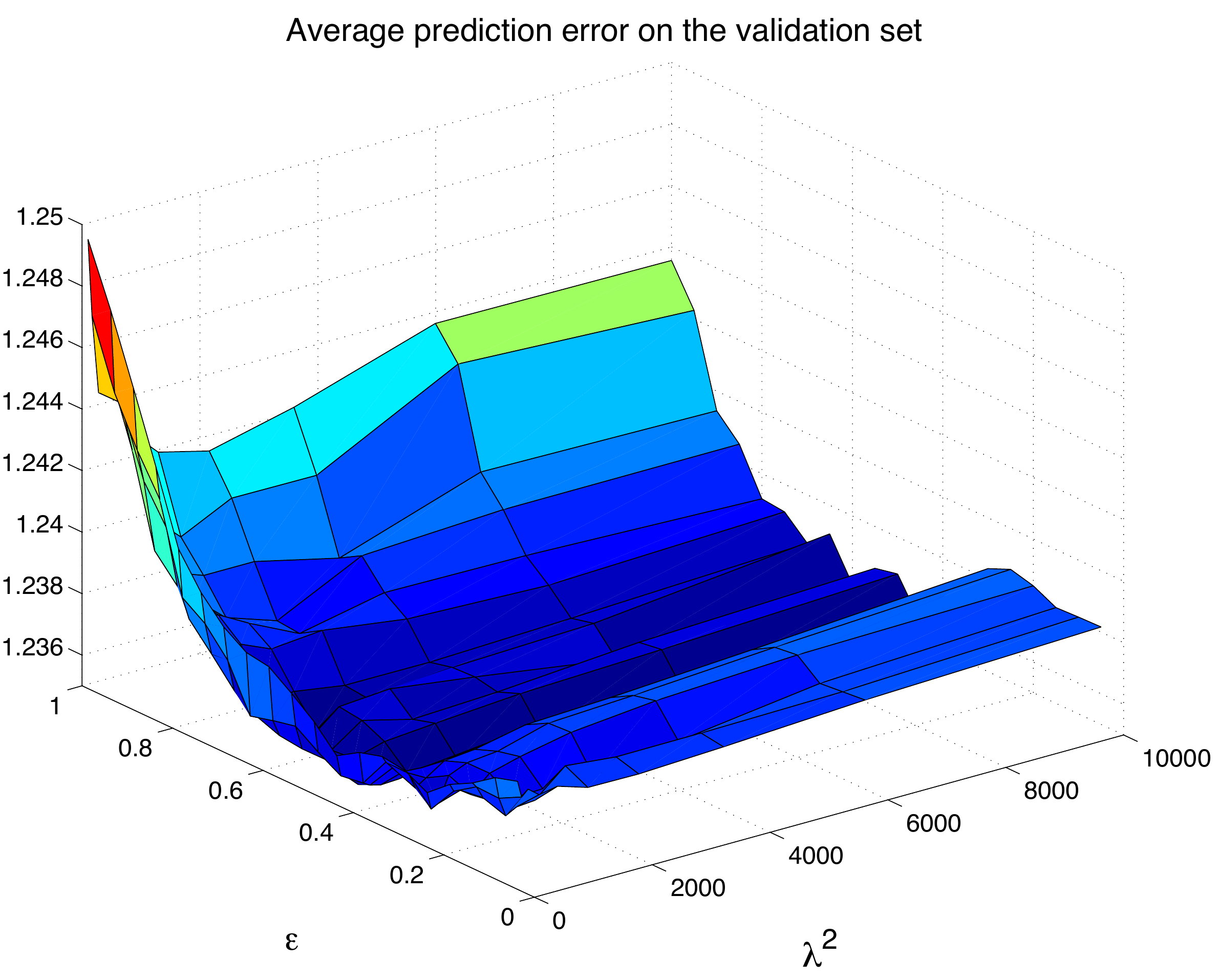}}
\end{center}
\caption{
    \label{FigVal}
    Estimation of the smoothing filter parameters using the Vapnik loss. 
    Average prediction error on the validation data set
    as a function of the variance process $\lambda^2$
    and $\epsilon$.
}
\end{figure}

\subsection{Real industrial data}

Let us now consider real industrial data coming from Syncrude Canada Ltd, also 
analyzed in~\cite{HLiu2004}. Oil production data is typically a multivariate 
time series capturing variables such as flow rate, pressure, particle velocity, and 
other observables. Because the data is proprietary, the exact nature of
the variables is not known. 
The data from~\cite{HLiu2004} comprises 
two anonymized time series variables, 
called V14 and V36, that have been selected from the process data. 
Each time series consists of 936 measurements, collected at times
$[1,2,\ldots,936]$ (see the top panels of Fig. \ref{Real}).  
Due to the nature of production data, we hypothesize 
that the temporal profile of the variables is 
smooth and that the observations contain outliers, 
as suggested by the fact that some observations differ markedly 
from their neighbors, especially in the case of V14.\\  
Our aim is to compare the prediction performance of two smoothers
that rely on $\ell_2$ and $\ell_1$ measurement loss functions.  
For this purpose, we consider 100 Monte Carlo runs. 
During each run, data are randomly divided into three disjoint sets:
 training and a validation data sets, both of size 350, and a test set of size 236.
We use the same state space model adopted in the previous subsection,
with $\Delta t = 1$, and use a non-informative 
prior to model the initial condition of the system.
The regularization parameter $\gamma$ (equal to the inverse of $\lambda^2$
assuming that the noise variance is 1) is chosen using standard cross validation 
techniques. For each value of $\gamma$, logarithmically spaced between $0.1$ and $1000$ (30 point grid), the smoothers are trained on the training set, 
and the $\gamma$ chosen corresponds to the smoother 
that achieves the best prediction on the validation set. 
After estimating $\gamma$, the variable's profile is reconstructed 
for the entire time series (at all times $[1,2,\ldots,936]$), using
the measurements contained in the union of the training and the validation data sets.  
Then, the prediction capability of the smoothers is evaluated by 
computing the 236 relative percentage errors (ratio of residual and observation times 100) 
in the reconstruction of the test set.\\
In Fig.~\ref{Real} we display the boxplots of the overall 23600 relative errors 
stored after the 100 runs for V14 (bottom left panel) and V36 (bottom right panel).
One can see that the $\ell_1$-Kalman smoother outperforms the classical one,
especially in case of V14. 
This is not surprising, since in this case prediction is more difficult 
due to the larger numbers of outliers in the time series.
In particular, for V14, the average percentage errors are $1.4\%$ and $2.4\%$ 
while, for V36, they are $1\%$ and $1.2\%$ using $\ell_1$ and $\ell_2$,
respectively. 

\begin{figure*}
  \begin{center}
   \begin{tabular}{cc}
\hspace{.1in}
 { \includegraphics[scale=0.36]{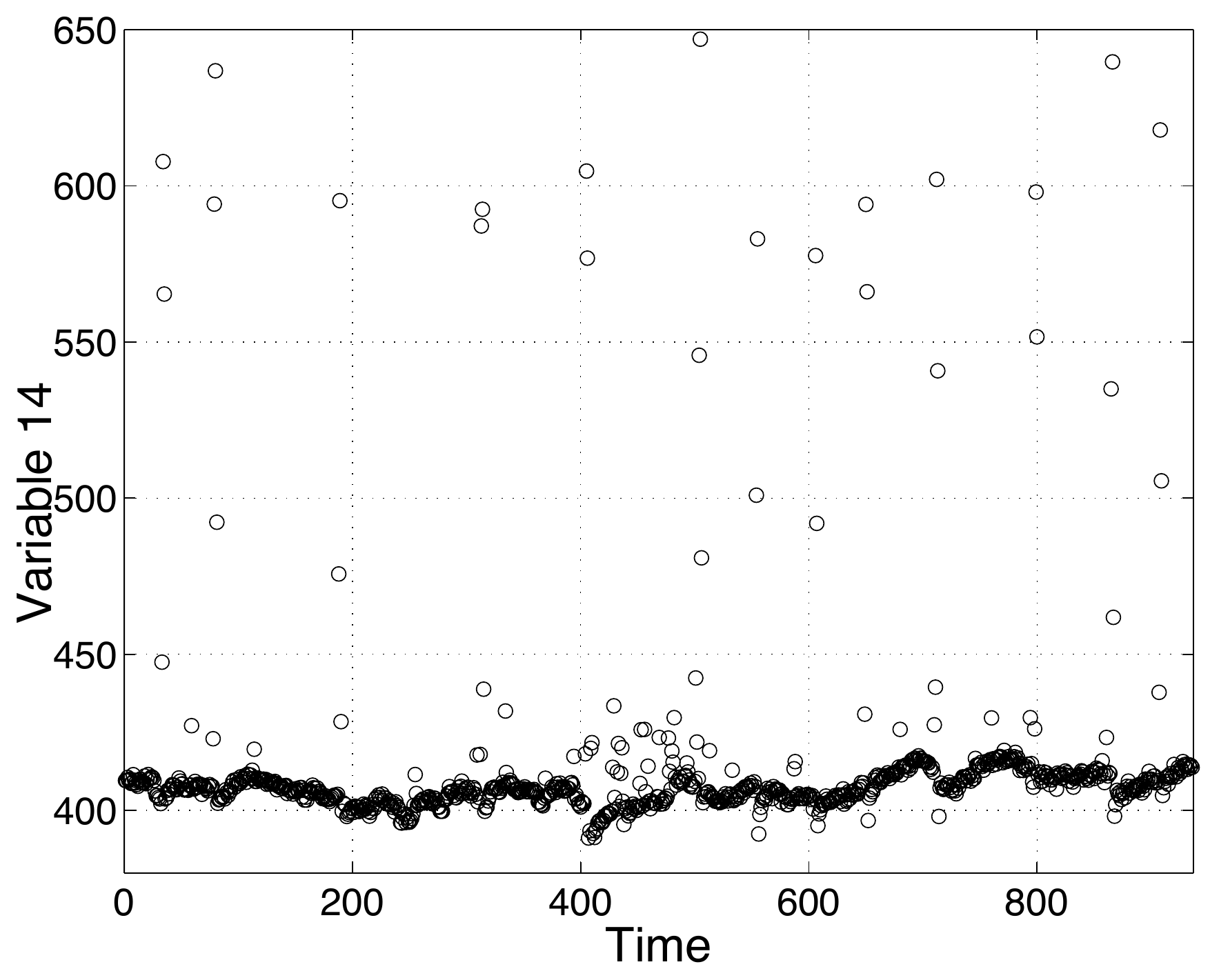}}
\hspace{.1in}
 { \includegraphics[scale=0.36]{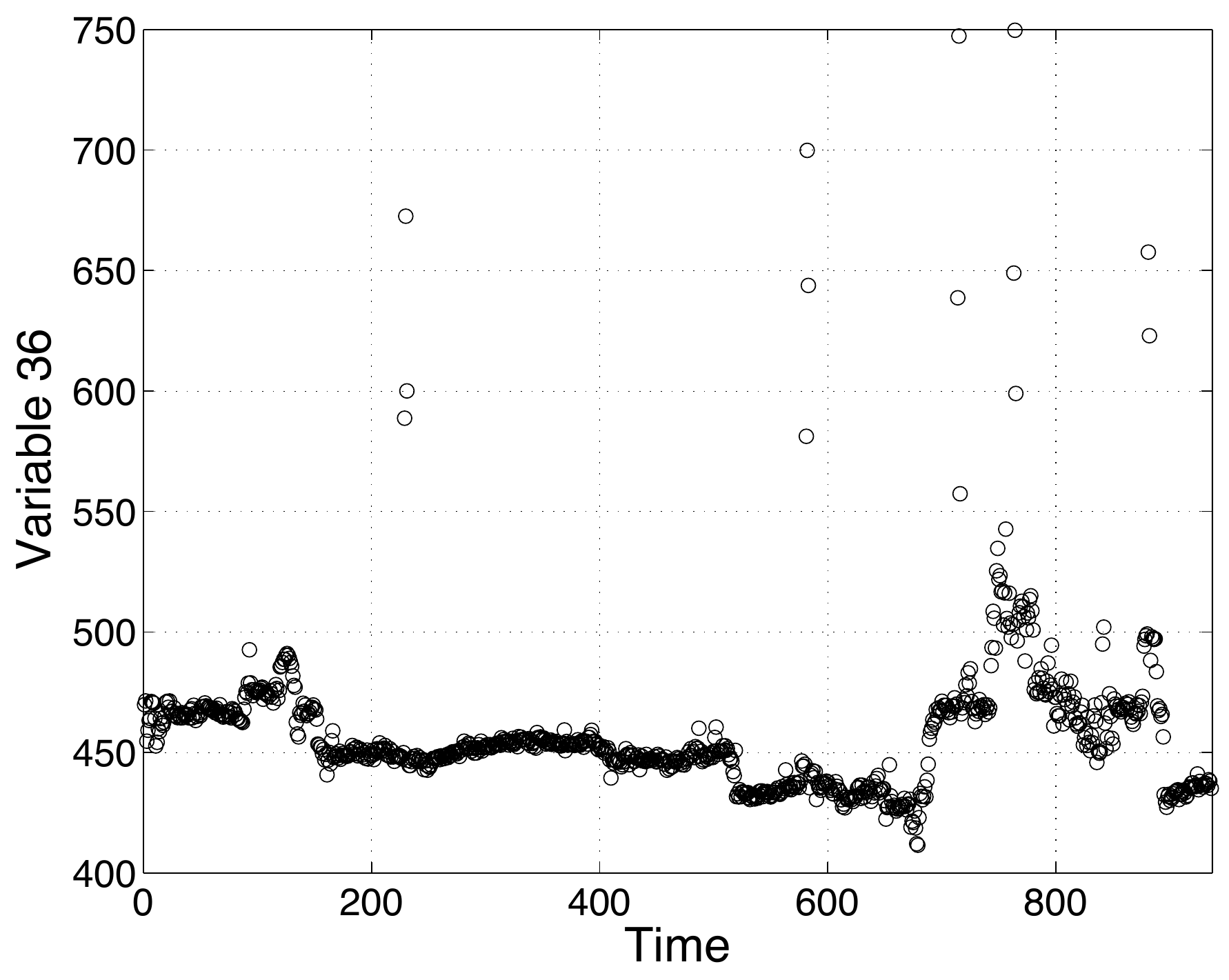}} \\
\hspace{.1in}
{\includegraphics[scale=0.36]{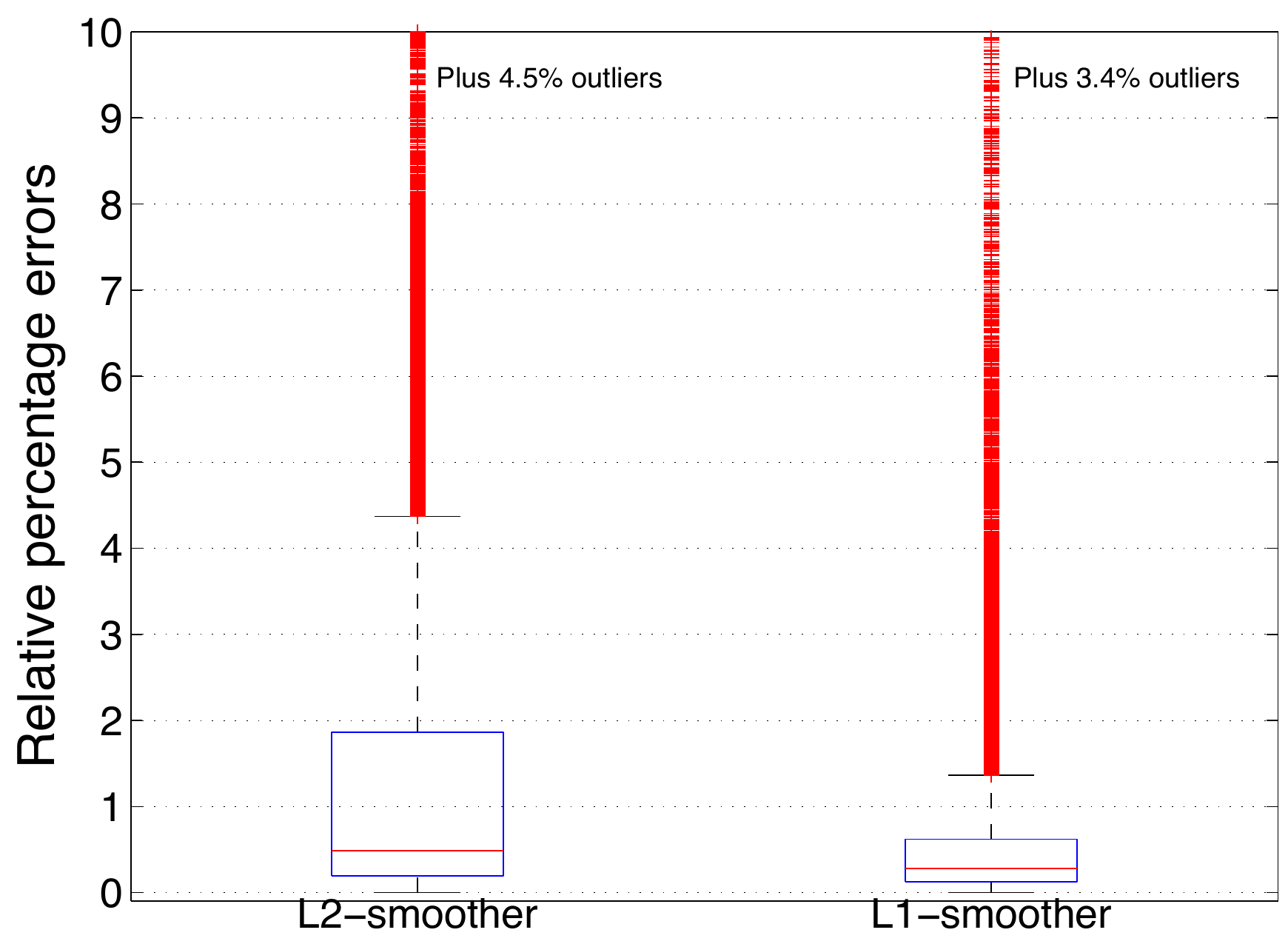}}
\hspace{.1in}
{\includegraphics[scale=0.36]{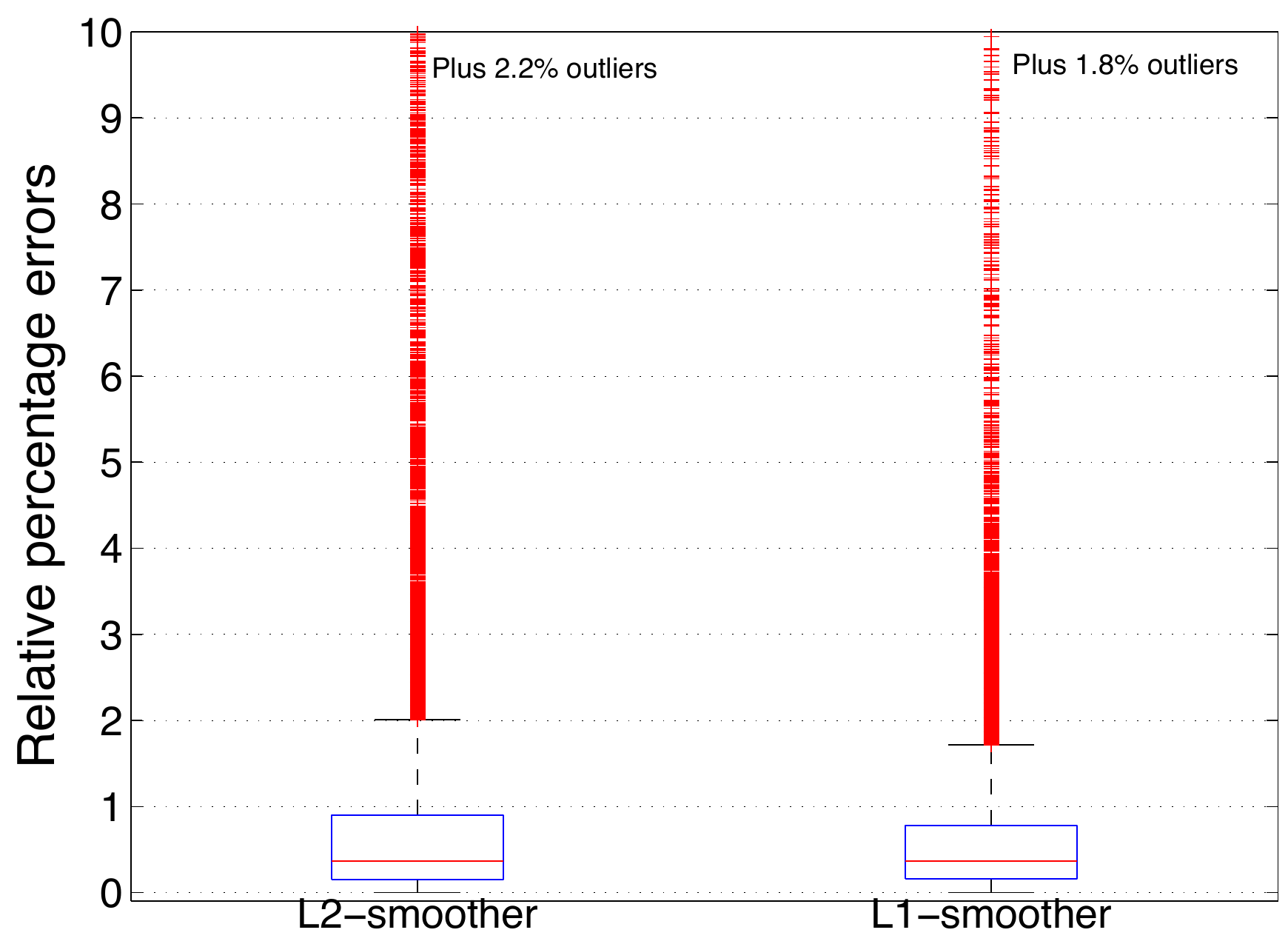}}
    \end{tabular}
    \caption{{\it{Left panels}}: data set for variable 14 (top) and relative percentage errors in the reconstruction of the test set obtained by  Kalman smoothers based on the $\ell_2$ and the $\ell_1$ loss (bottom). {\it{Right panels}}: data set for variable 36 (top) and relative percentage errors in the reconstruction of the test set obtained by  Kalman smoothers based on the $\ell_2$ and the $\ell_1$ loss (bottom).} \label{Real}
     \end{center}
\end{figure*}

\section{Conclusions}
\label{Conclusions}

\textcolor{black}{We have presented a new theory for robust and sparse estimation
using nonsmooth QS penalties. 
We give both primal and dual representations for these densities and show how
to obtain closed form expressions using Euclidean projections.
Using their dual representation, 
 we first derived conditions allowing the interpretation 
 of QS penalties as negative logs of
true probability densities, thus establishing 
  a statistical modeling framework.}
\textcolor{black}{In this regard, the coercivity condition
characterized in Th. \ref{coerciveRho} played a
central role. This  condition, necessary for the statistical 
interpretation, underscores the importance of 
an idea already useful in machine learning. 
Specifically, coercivity of the objective (\ref{probTwo}) is 
 a fundamental prerequisite in sparse and
robust estimation, as it precludes directions for which the sum of
the loss and the regularizer are insensitive to large parameter
changes. }Thus, the condition for a QS penalty to be a negative
log of a true density also ensures that the problem is well posed
in the machine learning context, i.e.  the learning machine 
has enough control over model complexity.\\
The QS class captures a variety of existing penalties when 
used either as a misfit measure or as a regularization functional. 
\textcolor{black}{
We have also shown how
to construct natural generalizations of these penalties within the QS class
that are based on general norm and cone
geometries.  Moreover, we show how the structure of these functions
can be understood through the use of Euclidean projections.}
\textcolor{black}{Moreover,
it is straightforward to use the presented results to design new formulations. 
Specifically, starting with the requisite {\it shape} of a new
penalty, one can use results of Section~\ref{PLQPTwo} to obtain a standardized
corresponding density, as well as the data structures $U, M, B, b$ required
to formulate and solve the optimization problem in Section~\ref{Optimization}. 
The statistical interpretation for these methods allows us to  
prescribe the mean and variance parameters of the corresponding model.}
\\
In the second part of the paper, we presented a broad computational
approach to solving estimation problems (\ref{probTwo}) 
using interior point methods.
\textcolor{black}
{In the process, we derived additional conditions that
guarantee the successful implementation of IP 
methods to compute the estimator (\ref{probTwo}) when $x$ and $v$ come from PLQ densities (a broad subclass of QS penalties), 
and provided a theorem characterizing the convergence of IP methods for this class. }
\textcolor{black}{
The key condition required for the successful execution of
IP iterations was a requirement on PLQ penalties to be finite
valued, which implies non-degeneracy of the corresponding
statistical distribution (the support cannot be contained in a
lower-dimensional subspace). The statistical interpretation is thus
strongly linked to the computational procedure.}\\
\textcolor{black}{We applied both the statistical framework and the computational
approach to the broad class of state
estimation problems in discrete-time dynamic systems, 
extending the classical formulations to allow dynamics and measurement
noise to come from any PLQ densities. 
Moreover, we showed that the classical computational efficiency
results can be preserved when the general IP approach 
is used in the state estimation context; specifically, 
PLQ Kalman smoothing can always be performed with a number of operations
that is linear in the length of the time series, as in the quadratic
case.}
\textcolor{black}{ The computational framework presented therefore allows the broad
application of interior point methods to a wide class of smoothing
problems of interest to practitioners. The powerful algorithmic
scheme designed here, together with the breadth and significance of
the new statistical framework presented, underscores the practical
utility and flexibility of this approach. We believe that this
perspective on modeling, robust/sparse estimation and
Kalman smoothing will be useful in a number of applications in the
years ahead.}\\
\textcolor{black}{While we only considered convex formulations in this paper, it is important to note
that the presented approach makes it possible to solve a much broader class of non-convex
problems. In particular, if the functions $Hx$ and $Gx$ in~\eqref{probTwo} are replaced
by nonlinear functions $g(x)$ and $h(x)$, the methods in this paper can be used
to compute descent directions for the non-convex problem. }
For an example of this approach, see~\citep{Aravkin2011tac}, which considers non-convex 
Kalman smoothing problems with nonlinear process and measurement models
and solves by
using the standard methodology of convex composite optimization \cite{Burke85}.
As in the Gauss-Newton method, 
at each outer iteration the process and measurement models are linearized around the current
iterate, and the descent direction is found by solving 
a particular subproblem of type~\eqref{probTwo} using IP methods.\\
In many contexts, it would be useful to estimate the parameters 
that define QS penalties; for example the $\kappa$ in the Huber penalty
or the $\epsilon$ in the Vapnik penalty. In the numerical examples 
presented in this paper, we have relied on cross-validation to accomplish this task. 
An alternative method could be to compute the MAP points returned by our estimator 
for different filter parameters to gain information about the joint posterior of states and parameters.  
This strategy could help in designing a good proposal density for posterior simulation using 
e.g. particle smoothing filters \citep{Ristic}. 
We leave a detailed study of this approach to the QS modeling framework for future work.  

\section{Appendix}
\label{Appendix}
%
%
\subsection{Proof of Theorem \ref{domainCharTheorem}}
Let $\rho(y)=\rho(U,M,I,0;y)$ so that $\rho(U,M,B,b;y)=\rho(b+By)$. Then
$\dom{\rho(U,M,B,b;\cdot)}=B^{-1}(\dom{\rho}-b)$, hence the theorem follows if it
can be shown that  $\mathrm{bar}(U)+\Ran{M}\subset\dom{\rho}\subset [U^\infty\cap\R{null}(M)]^\circ$
with equality when $\mathrm{bar}(U)+\Ran{M}$ is closed. 
Observe that if there exists 
$w\in U^\infty\cap\R{null}(M)$ such that $\ip{y}{w}>0$, then trivially $\rho(y)=+\infty$ so 
$y\notin\dom{\rho}$. 
Consequently, $\dom{\rho}\subset [U^\infty\cap\R{null}(M)]^\circ$.
Next let $y\in \mathrm{bar}(U)+\Ran{M}$, then there is a $v\in \mathrm{bar}(U)$ and $w$ such that
$y=v+Mw$. Hence
\begin{eqnarray*}
\sup_{u\in U}[\ip{u}{y}-\half\ip{u}{Mu}]
&=&
\sup_{u\in U}[\ip{u}{v+Mw}-\half\ip{u}{Mu}]
\\ &=&
\sup_{u\in U}[\ip{u}{v}+\half w^TMw-\half(w-u)^TM(w-u)]
\\ &\le &
\support{v}{U} +\half w^TMw \ <\ \infty\ .
\end{eqnarray*}
Hence $\mathrm{bar}(U)+\Ran{M}\subset \dom{\rho}$.

If the set $\mathrm{bar}(U)+\Ran{M}$ is closed, then so is the set $\mathrm{bar}(U)$.
Therefore, by \cite[Corollary 14.2.1]{RTR}, $(U^\infty)^\circ =\mathrm{bar}(U)$, and, by
\cite[Corollary 16.4.2]{RTR}, $[U^\infty\cap\R{null}(M)]^\circ=\mathrm{bar}(U)+\Ran{M}$,
which proves the result.


The polyhedral case $\mathrm{bar}(U)$ is a polyhedral convex set, and the sum of
such sets is also a polyhedral convex set \cite[Corollary 19.3.2]{RTR}.
\subsection{Proof of Theorem \ref{QS representations}}
To see the first equation in \eqref{QS primal} write
\(
\rho(y)=\sup_{u}\left[\ip{y}{u}-\left(\half\Vert L^Tu\Vert_2^2+\indicator{u}{U}\right)\right]\ ,
\)
and then apply the calculus of convex conjugate functions \cite[Section 16]{RTR} to find that
\[
\left(\half\Vert L^T\cdot\Vert_2^2+\indicator{\cdot}{U}\right)^*(y)=
\inf_{s\in\mB{R}^k}\left[ \half\Vert s\Vert^2_2+\support{y-Ls}{U}\right]\ .
\]
The second equivalence in \eqref{QS primal} follows from \cite[Theorem 14.5]{RTR}.

For the remainder, we assume that $M$ is positive definite.
In this case it is easily shown that $(MU)^\circ=M^{-1}U^\circ$. Hence, by
\cite[Theorem 14.5]{RTR}, $\gauge{\cdot}{MU}=\support{\cdot}{M^{-1}U^\circ}$.
We use these facts freely throughout the proof.

The formula \eqref{QS primal M} follows by observing that
\[
\half\Vert s\Vert^2_2+\support{y-Ls}{U}=
\half \Vert L^{-T}s\Vert^2_M+\support{M^{-1}y-L^{-T}s}{MU}
\]
and then making the substitution $v=L^{-T}s$.
To see \eqref{QS primal M 2}, note that the optimality conditions
for \eqref{QS primal M} are $Ms\in\partial \support{M^{-1}y-s}{MU}$, or equivalently,
$M^{-1}y-s\in\ncone{Ms}{MU}$, i.e. $s\in U$ and
\[
\ip{M^{-1}y-s}{u-s}_M=\ip{M^{-1}y-s}{M(u-s)}\le 0\ \forall\ u\in U,
\]
which, by \eqref{proj oc}, tells us that $s=\proj{M^{-1}y}{U}{M}$.
Plugging this into \eqref{QS primal M} gives \eqref{QS primal M 2}.

Using the substitution $v=Ls$, 
the argument showing \eqref{QS primal M inv} and \eqref{QS primal M inv 2}
differs only slightly from that for \eqref{QS primal M} and \eqref{QS primal M}
and so is omitted.

The formula \eqref{QS U proj} follows by completing the square in the $M$-norm in 
the definition \eqref{PLQpenalty}:
\begin{eqnarray*}
\ip{y}{u}-\half \ip{u}{Mu}
&=&
\ip{M^{-1}y}{u}_M-\half \ip{u}{u}_M
\\ &=&
\half y^TM^{-1}y-\half[\ip{M^{-1}y}{M^{-1}y}_M-2\ip{M^{-1}y}{u}_M + \ip{u}{u}_M]
\\ &=&
\half y^TM^{-1}y-\half\| M^{-1}y-u\|_M^2\ .
\end{eqnarray*}
The result as well as \eqref{QS U proj 2} now follow from Theorem \ref{projection theorem}.
Both \eqref{QS MU proj} and \eqref{QS MU proj 2} follow similary by completing the
square in the $M^{-1}$-norm. 

\subsection{Proof of Theorem \ref{PLQIntegrability}}

First we will show that if $\rho$ is convex coercive, 
then for any $\bar x \in \argmin f \neq \emptyset$, 
there exist constants $R$ and $K > 0$ such that
\begin{equation}
\label{boundRho}
\rho(x) \geq \rho(\bar x) + K \|x - \bar x\| \quad \forall \; x \notin R\mB{B}\;.
\end{equation}

Without loss of generality, we can assume that $0 = \rho(0) = \inf \rho$. 
Otherwise, replace $\rho(x)$ by $\hat \rho (x) = \rho(x + \bar x) - \rho(\bar x)$, 
where $\bar x$ is any global minimizer of $\rho$.

Let $\alpha > 0$. Since $\rho$ is coercive, there exists $R$ such that $\lev{\rho}{\alpha} \subset R\mB{B}$. 
We will show that $\frac{\alpha}{R}\|x\| \leq \rho(x)$ for all $x \notin R\mB{B}$. 

Indeed, for all $x\ne 0$, we have $\rho(\frac{R}{\|x\|}x) \geq \alpha$. Therefore,
if  $x \notin R\mB{B}$, then $0 < \frac{R}{\|x\|} < 1$, and  we have 
\[
\frac{\alpha}{R}\|x\| \leq \frac{\|x\|}{R}\rho\left(\frac{R}{\|x\|}x\right) \leq \frac{\|x\|}{R}\frac{R}{\|x\|} \rho(x) = \rho(x).
\] 

Then by~\eqref{boundRho},
\[
\begin{aligned}
\int \exp(-\rho(x)) dx &= \int_{\bar x + R\mB{B}} \exp(-\rho(x)) dx  + \int_{\|x- \bar x\| > R}  \exp(-\rho(x)) dx\\
& \leq C_1 +  C_2\int_{\|x- \bar x\| > R}  \exp(-K\|x - \bar x\|) dx < \infty\;.
\end{aligned}
\]

\subsection{Proof of Theorem \ref{coerciveRho}}

First observe that $B^{-1}\polar{[\mathrm{cone}(U)]} =
\polar{[B^\R{T}\mathrm{cone}(U)]}$ by \citep[Corollary 16.3.2]{RTR}.

Suppose that $\hat y \in B^{-1} \polar{[\R{cone}(U)]}$, and $\hat y
\neq 0$. Then $B\hat y \in \mathrm{cone}(U)$, and $B\hat y \neq 0$
since $B$ is injective, and we have
\[
\begin{array}{lll}
\rho(t \hat y) &=& \sup_{u \in U} \langle b + t B \hat y,
u\rangle -
\frac{1}{2}u^\R{T}M u  \\
&=& \sup_{u \in U} \langle b , u\rangle - \frac{1}{2}u^\R{T}M u +
t  \langle B\hat y, u \rangle
\\
&\leq & \sup_{u \in U} \langle b , u\rangle -
\frac{1}{2}u^\R{T}M u \\
&\leq & \rho(U, M, 0, I; b),
\end{array}
\]
so $\rho(t \hat y)$ stays bounded even as $t \rightarrow
\infty$, and so $\rho$ cannot be coercive.

Conversely, suppose that $\rho$ is not coercive. Then we can find a
sequence $\{y_k\}$ with $\|y_k\| > k$ and a constant $P$ so that
$\rho(y_k) \leq P$ for all $k > 0$. Without loss of generality, we
may assume that $\frac{y_k}{\|y_k\|}\rightarrow \bar y$.

Then by definition of $\rho$, we have for all $u \in U$
\[
\begin{array}{lll}
&\langle b + By_k, u \rangle - \frac{1}{2}u^\R{T}Mu \leq P\\
& \langle b + By_k, u \rangle \leq P + \frac{1}{2}u^\R{T} M u\\
& \langle \frac{b + By_k}{\|y_k\|}, u \rangle \leq \frac{P}{\|y_k\|}
+ \frac{1}{2\|y_k\|}u^\R{T} M u
\end{array}
\]
Note that $\bar y \neq 0$, so $B \bar y \neq 0$. When we take the
limit as $k \rightarrow \infty$, we get $\langle B\bar y, u \rangle
\leq 0$. From this inequality we see that $B \bar y \in \polar{[\R{cone}(U)]}$, 
and so $\bar y \in B^{-1}\polar{[\R{cone}(U)]}$.

\subsection{Proof of Theorem \ref{IPMtheorem}}
\begin{proof}
(i) 
Using standard elementary row operations, reduce the matrix
\begin{equation}
\label{rKKTder}
F_\gamma^{(1)}
:=
\begin{bmatrix}
I & 0 & A^\R{T} & 0\\
D(q) & D(s) & 0 & 0 \\
0 & -A & -M & B\\
0 & 0 & B^\R{T} & 0
\end{bmatrix}\;
\end{equation}
to
\[
\begin{bmatrix}
I & 0 & A^\R{T} & 0\\
0 & D(s) & -D(q)A^\R{T} & 0 \\
0 & 0 & -T & B\\
0 & 0 & B^\R{T} & 0
\end{bmatrix}\;,
\]
where
$T = M + AD(q)D(s)^{-1}A^\R{T}$.
The matrix $T$ is invertible since
$\R{null}(M)\cap \R{null}(C^\R{T}) =\{0\}$.
Hence, we can further reduce this matrix to the block upper triangular form
\[
\begin{bmatrix}
I & 0 & A^\R{T} & 0\\
0 & D(s) & -D(q)C^\R{T} & 0 \\
0 & 0 & -T & B\\
0 & 0 & 0 & -B^\R{T} T^{-1}B
\end{bmatrix}\;.
\]
Since $B$ is injective, the matrix $B^\R{T} T^{-1}B$ is also invertible. 
Hence this final block upper triangular is invertible proving
Part (i).
\smallskip

\noindent
(ii) Let $(s,q)\in\widehat{\mathcal{F}}_+$ and choose $(u_i,y_i)$ so that
$(s,q,u_i,y_i)\in \mathcal{F}_+$ for $i=1,2$. Set $u:=u_1-u_2$ and $y:=y_1-y_2$.
Then, by definition, 
\begin{equation}\label{null space contradiction}
0=A^\R{T}u,\ 0=By-Mu,\mbox{ and }0=B^\R{T}u\ .
\end{equation}
Multiplying the second of these equations on the left by $u$ and utilizing
the third as well as the positive semi-definiteness of $M$, we find that
$Mu=0$. Hence, $u\in \R{null}(M)\cap \R{null}(A^\R{T}) =\{0\}$, and so $By=0$.
But then $y=0$ as $B$ is injective.
\smallskip

\noindent
(iii) 
Let $(\hat s,\hat q,\hat u,\hat y)\in\mathcal{F}_+$ and 
$(s,q,u, y)\in\mathcal{F}_+(\tau)$. Then, by \eqref{fullKKT},
\begin{eqnarray*}
(s-\hat s)^\R{T}(q-\hat q)&=& [(a-A^\R{T}u)-(a-A^\R{T}\hat u)]^\R{T}(q-\hat q)\\
&=&(\hat u-u)^\R{T}(Aq-A\hat q)\\
&=&(\hat u-u)^\R{T}[(b+By-Mu)-(b+B\hat b-M\hat u)]\\
&=&(\hat u-u)^\R{T}M(\hat u-u)\\
&\ge&0.
\end{eqnarray*}
Hence,
\[
\tau+\hat s^\R{T}\hat q\ge s^\R{T}y+\hat s^\R{T}\hat q\ge s^\R{T}\hat y+y^\R{T}\hat s\ge \xi\norm{(s,q)}_1,
\]
where $\xi=\min\set{\hat s_i,\ \hat q_i}{i=1,\dots,\ell}>0$.
Therefore, the set
\[
\widehat{\mathcal{F}}_+(\tau)=\set{(s,q)}{(s,q,u,y)\in \mathcal{F}_+(\tau)}
\] 
is bounded.
Now suppose the set $\mathcal{F}_+(\tau)$ is not bounded. Then there exits
a sequence $\{(s_\nu,q_\nu,u_\nu,y_\nu)\}\subset \mathcal{F}_+(\tau)$ such that
$\norm{(s_\nu,q_\nu,u_\nu,y_\nu)}\uparrow +\infty$. Since $\widehat{\mathcal{F}}_+(\tau)$
is bounded, we can assume that $\norm{(u_\nu,y_\nu)}\uparrow +\infty$ while
$\norm{(s_\nu,q_\nu)}$ remains bounded. With no loss in generality, we may assume
that there exits $(u,y)\ne (0,0)$ such that $ (u_\nu,y_\nu)/\norm{(u_\nu,y_\nu)}\rightarrow (u,y)$.
By dividing \eqref{fullKKT} by $\norm{(u_\nu,y_\nu)}$ and taking the limit, we find that
\eqref{null space contradiction} holds. But then, as in \eqref{null space contradiction},
$(u,y)=(0,0)$. This contradiction yields the result.
\smallskip

\noindent
(iv) 
We first show existence.
This follows from a standard continuation argument.
Let $(\hat s,\hat q,\hat u,\hat y)\in\mathcal{F}_+$ and $v\in\mB{R}^\ell_{++}$. 
Define
\begin{equation}
\label{homotopy F}
F(s, q, u, y,t)
=
\begin{bmatrix}
s + A^\R{T}u - a \\
D(q)D(s)\B{1} - [(1-t) \hat v +tv]\\
By - Mu - Aq \\
B^\R{T}u + b
\end{bmatrix}\; ,
\end{equation}
where $\hat g:=(\hat s_1\hat y_1,\dots,\hat s_\ell\hat y_\ell)^\R{T}$.
Note that 
\[F(\hat s,\hat q,\hat u,\hat y,0)=0\mbox{ and, by Part (i), }
\nabla_{(s, q, u, y)}F(\hat s,\hat q,\hat u,\hat y,0)^{-1}\mbox{ exists}.
\] 
The Implicit Function Theorem implies that there is a $\tilde t>0$ and a 
differentiable mapping $t\mapsto (s(t), q(t), u(t), y(t))$ on $[0,\tilde t)$
such that 
\[
F[s(t), q(t), u(t), y(t),t]=0\mbox{ on }[0,\tilde t).
\] 
Let
$\bar t>0$ be the largest such $\tilde t$ on $[0,1]$. 
Since
\[
\set{[s(t), q(t), u(t), y(t)]}{t\in [0,\bar t)}\subset\mathcal{F}_+(\bar \tau),
\]
where $\bar\tau=\max\{\B{1}^\R{T}\hat g,\B{1}^\R{T}g\}$, 
Part (iii) implies that there is a sequence $t_i\rightarrow \bar t$ and a point
$(\bar s, \bar q, \bar u, \bar y)$ such that 
$[s(t_i), q(t_i), u(t_i), y(t_i)]\rightarrow (\bar s, \bar q, \bar u, \bar y)$.
By continuity $F(\bar s, \bar q, \bar u, \bar y,\bar t)=0$. If $\bar t=1$, we are done;
otherwise, apply the Implicit Function Theorem again at 
$(\bar s, \bar q, \bar u, \bar y,\bar t)$ to obtain a contradiction to the 
maximality of $\bar t$.

We now show uniqueness. By Part (ii), we need only establish the uniqueness of $(s,q)$.
Let $(s^\nu,q^\nu)\in \widehat{\mathcal{F}}_+$ be such that 
$g=(s_{j(1)}q_{j(1)},s_{j(2)}q_{j(2)},\dots,s_{j(\ell)} q_{j(\ell)})^\R{T}$, where 
$s_{j(i)}$ denotes the $i$th element of $s_j$, and $j=1,2$.
As in Part (iii), we have $(s_1-s_2)^\R{T}(q_1-q_2)=(u_1-u_2)^\R{T}M((u_1-u_2)\ge 0$, and,
for each $i=1,\dots,\ell$, $s_{1(i)}q_{1(i)}=s_{2(i)}q_{2(i)}=g_i>0$. If $(s_1,q_1)\ne(s_2,q_2)$,
then, for some $i\in\{1,\dots,\ell\}$, $(s_{1(i)}-s_{2(i)})(q_{1(i)}-q_{2(i)})\ge 0$ and 
either $s_{1(i)}\ne s_{2(i)}$ or $q_{1(i)}\ne q_{2(i)}$. If $s_{1(i)}>s_{2(i)}$, then $q_{1(i)}\ge q_{2(i)}>0$
so that $g_i=s_{1(i)}q_{1(i)}>s_{2(i)}q_{2(i)}=g_i$, a contradiction.
So with out loss in generality (by exchanging $(s_1,q_1)$ with $(s_2,q_2)$ if
necessary), we must have $q_{1(i)}>q_{2(i)}$. But then $s_{1(i)}\ge s_{2(i)}>0$, so that again
$g_i=s_{1(i)}q_{1(i)}>s_{2(i)}q_{2(i)}=g_i$, and again a contradiction. Therefore, $(s,q)$ is unique.
\smallskip

\noindent
(v) Apply Part (iv) to get a point on the central path and then use the continuation 
argument to trace out the central path. The differentiability follows from the implicit function
theorem.
\smallskip

\noindent
(vi) Part (iii) allows us to apply a standard compactness argument to get the existence of
cluster points and the continuity of $F_\gamma(s,q,u,y)$ in all of its arguments including
$\gamma$ implies that all of these cluster points solve \eqref{fullKKT}.
%
\end{proof}

\subsection{Details for Remark \ref{propKS}}

The Lagrangian for (\ref{PLQsubproblem})
for feasible $(x, u_w, u_v)$ is
\begin{equation}
\label{PLQLagrangian}
\begin{aligned}
 L(x, u_w, u_v) &= \left\langle
\begin{bmatrix}
\tilde b_w
\\ \tilde b_v
\end{bmatrix},
\begin{bmatrix}
u_w
\\ u_v
\end{bmatrix}
\right\rangle - \frac{1}{2}
\begin{bmatrix}
u_w \\
u_v
\end{bmatrix}^\R{T}
\begin{bmatrix}
M_w & 0 \\
0 & M_v
\end{bmatrix}
\begin{bmatrix}
u_w \\
u_v
\end{bmatrix}
- \left\langle
\begin{bmatrix}
u_w\\
u_v
\end{bmatrix}\;,
\begin{bmatrix}
- B_wQ^{-1/2}G \\
B_vR^{-1/2}H
\end{bmatrix}
x \right\rangle\;
\end{aligned}
\end{equation}
where $\tilde b_w = b_w - B_wQ^{-1/2}\tilde x_0$ and
$\tilde b_v = b_v - B_vR^{-1/2}z$.
The associated optimality conditions for feasible $(x, u_w, u_v)$
are given by
\begin{equation}
\label{PLQOptimalityConditions}
\begin{array}{lll}
&G^\R{T}Q^{-\R{T}/2}B_w^\R{T}\bar u_w - H^\R{T}R^{-\R{T}/2}B_v^\R{T}\bar u_v  = 0\\
&\tilde b_w - M_w \bar{u}_w + B_w Q^{-1/2}G\bar{x} \in N_{U_w}(\bar{u}_w)\\
&\tilde b_v - M_v \bar{u}_v - B_v R^{-1/2}H\bar{x} \in
N_{U_v}(\bar{u}_v)\;,
\end{array}
\end{equation}
where $N_C(r)$ denotes the normal cone to the set $C$ at the point $r$
(see \citep{RTR} for details).

Since $U_w$ and $U_v$ are polyhedral, we can derive
explicit representations of the normal cones $N_{U_w}(\bar u_w)$ and $N_{U_v}(\bar u_v)$.
For a polyhedral set $U \subset \mB{R}^m$
and any point $\bar{u} \in U$, the normal cone $N_U(\bar{u})$ is
polyhedral. Indeed, relative to any representation
\[
U = \{u|A^\R{T}u \leq a\}
\]
and the active index set \( I(\bar{u}) := \{i| \langle A_{i},
\bar{u}\rangle = a_i\} \), where $A_i$ denotes the $i$th column of
$A$, we have
\begin{equation}
\label{NormalRep}
N_U(\bar{u}) = \left\{
\begin{aligned}
q_1A_1 + \dots + q_mA_m\;  | \;q_i \geq 0\;
\R{for} \; & i\in I(\bar{u})\\
 q_i = 0 \; \R{for}\;  & i \not\in I(\bar{u})
\end{aligned}
 \right\} .
\end{equation}
Using~\eqref{NormalRep},
 Then  we may rewrite the optimality conditions
(\ref{PLQOptimalityConditions}) more explicitly as
\begin{equation}
\label{PLQExpandedConditions}
\begin{aligned}
&G^\R{T}Q^{-\R{T}/2}B_w^\R{T}\bar u_w -
H^\R{T}R^{-\R{T}/2}B_v^\R{T}\bar u_v  = 0\\
& \tilde b_w - M_w \bar{u}_w + B_w Q^{-1/2}G\bar{d} = A_wq_w\\
&\tilde b_v - M_v \bar{u}_v - B_v R^{-1/2}H\bar{d} = A_v q_v \\
&\{q_v \geq 0 | q_{v(i)} = 0\; \R{for}\; i \not\in I(\bar u_v)\}\\
&\{q^w \geq 0 | q_{w(i)} = 0\; \R{for}\; i \not\in I(\bar u_w)\}\;\\
\end{aligned}
\end{equation}
where $q_{v(i)}$ and $q_{w(i)}$ denote the $i$th elements of $q_v$ and $q_w$.
Define slack variables $s_w \geq 0$ and $s_v \geq 0$ as follows:
\begin{equation}
\label{slack}
\begin{array}{lll}
s_w &=&  a_w - A_w^\R{T}u_w\\
s_v &=&  a_v - A_v^\R{T}u_v .
\end{array}
\end{equation}
Note that we know the entries of $q_{w(i)}$ and $q_{v(i)}$ are zero if and
only if the corresponding slack variables $s_{v(i)}$ and $s_{w(i)}$ are
nonzero, respectively. Then we have $q_w^\R{T}s_w = q_v^\R{T}s_v
= 0$. These equations are known as the complementarity conditions.
Together, all of these equations give system~\eqref{PLQFinalConditions}.
%

\subsection{Proof of Theorem \ref{thmPLQsmoother}}
IP methods apply a damped Newton iteration to
find the solution of the  relaxed KKT system $F_{\gamma} = 0$,
where
\[
F_{\gamma} \left(
\begin{matrix}
s_w\\
s_v\\
q_w\\
q_v\\
u_w\\
u_v\\
x
\end{matrix}
\right) =
\begin{bmatrix}
A_w^\R{T}u_w + s_w - a_w\\
A_v^\R{T}u_v + s_v - a_v\\
D(q_w)D(s_w)\B{1} - \gamma\B{1}\\
D(q_v)D(s_v)\B{1} - \gamma\B{1}\\
\tilde b_w + B_w Q^{-1/2}Gd -  M_w u_w - A_wq_w\\
\tilde b_v - B_v R^{-1/2}Hd - M_v u_v -  A_v q_v\\
G^\R{T}Q^{-\R{T}/2}B_w^\R{T} u_w -
H^\R{T}R^{-\R{T}/2}B_v^\R{T}\bar u_v
\end{bmatrix}.
\]
This entails solving the system
\begin{equation}
\label{NewtonSystem} F_{\gamma}^{(1)} \left(
\begin{matrix}
s_w\\
s_v\\
q_w\\
q_v\\
u_w\\
u_v\\
x
\end{matrix}
\right)
\begin{bmatrix}
\Delta s_w\\
\Delta s_v\\
\Delta q_w\\
\Delta q_v\\
\Delta u_w\\
\Delta u_v\\
\Delta x
\end{bmatrix}
= -F_{\gamma} \left(
\begin{matrix}
s_w\\
s_v\\
q_w\\
q_v\\
u_w\\
u_v\\
x
\end{matrix}
\right),
\end{equation}
where  the derivative matrix $F_{\gamma}^{(1)}$ is given by
\begin{equation}
\label{Fprime}
\begin{bmatrix}
I & 0 & 0 & 0 & (A_w)^\R{T} & 0 & 0\\
0 & I & 0 & 0 & 0 & (A_v)^\R{T} & 0\\
D(q_w) & 0 & D(s_w) & 0 & 0 & 0 & 0\\
0 & D(q_v) & 0 & D(s_v) & 0 & 0 & 0 \\
0 & 0 & - A_w & 0 & -M_w & 0 & B_wQ^{-1/2}G \\
0 & 0 & 0 & -A_v & 0 & -M_v & -B_vR^{-1/2}H \\
0 & 0 & 0 & 0 & G^\R{T}Q^{-\R{T}/2}B_w^\R{T} &
-H^\R{T}R^{-\R{T}/2}B_v^\R{T} & 0
\end{bmatrix}
\end{equation}
We now show the row operations necessary to reduce the matrix
$F_{\gamma}^{(1)}$ in (\ref{Fprime})
to upper block triangular form. After each
operation, we show only the row that was modified.
\begin{equation*}
\begin{array}{lll}
&\R{row}_3 \leftarrow \R{row}_3 - D(q_w)\;\R{row}_1\\
&\begin{bmatrix}
0 & 0 & D(s_w) & 0 & -D(q_w)A_w^\R{T} & 0 & 0\\
\end{bmatrix}\\
&\R{row}_4 \leftarrow \R{row}_4 - D(q_v)\;\R{row}_2 \\
&\begin{bmatrix}
0 & 0 & 0 & D(s_v) & 0 & -D(q_v)A_v^\R{T} & 0 \\
\end{bmatrix}\\
&\R{row}_5 \leftarrow \R{row}_5 + A_wD(s_w)^{-1}\;\R{row}_3\\
&\begin{bmatrix}
0 & 0 & 0 & 0 & -T_w & 0 & B_wQ^{-1/2}G \\
\end{bmatrix}\\
&\R{row}_6 \leftarrow \R{row}_6 + A_vD(s_v)^{-1}\;\R{row}_4 \\
&\begin{bmatrix}
0 & 0 & 0 & 0 & 0 & -T_v  & -B_vR^{-1/2}H
\end{bmatrix}\;.
\end{array}
\end{equation*}
In the above expressions,
\begin{equation}
\label{KalmanT}
\begin{aligned}
T_w &:=  M_w + A_wD(s_w)^{-1}D(q_w)A_w^\R{T}\\
T_v &:= M_v + A_vD(s_v)^{-1}D(q_v)A_v^\R{T}\;,
\end{aligned}
\end{equation}
where $D(s_w)^{-1}D(q_w)$ and $D(s_v)^{-1}D(q_v)$ are always
full-rank diagonal matrices, since the vectors $s_w, q_w, s_v, q_v$.
Matrices $T_w$ and $T_v$ are invertible as long as the PLQ densities for $w$
and $v$ satisfy~\eqref{NullProperty}.

%
%
\begin{remark}[block diagonal structure of $T$ in i.d.
case]
\label{TblockDiagonal} Suppose that ${y}$ is a random vector, ${y} =
\R{vec}(\{{y_k}\})$, where each ${y_i}$ is itself a
random vector in $\mB{R}^{m(i)}$, from some PLQ density 

\noindent
$\B{p}(y_i) \propto \exp[-c_2\rho(U_i, M_i, 0, I; \cdot)]$, and all $y_i$
are independent. Let $U_i = \{u: A_i^\R{T}u \leq a_i\}$. Then the matrix
$T_\rho$ is given by
\(
T_\rho = M + ADA^\R{T}
\)
where
$M = \R{diag}[M_1, \cdots, M_N]$,
$A = \R{diag}[A_1, \cdots,A_N]$,
$D = \R{diag}[D_1, \cdots,D_N]$,
and $\{D_i\}$ are diagonal with positive entries.
Moreover, $T_\rho$ is block diagonal, with $i$th diagonal block given by
$M_i + A_iD_iA_i^\R{T}$.
\end{remark}
From Remark~\ref{TblockDiagonal}, 
the matrices $T_w$ and $T_v$ in (\ref{KalmanT}) are block diagonal provided that
$\{w_k\}$ and $\{v_k\}$ are independent vectors from any PLQ
densities.

We now finish the reduction of $F_{\gamma}^{(1)}$ to upper block
triangular form:
\begin{equation*}
\begin{aligned}
 \R{row}_7 &\leftarrow \R{row}_7 +
\Big(G^\R{T}Q^{-\R{T}/2}B_w^\R{T}T_w^{-1}\Big)\R{row}_5 
-\Big(H^\R{T}R^{-\R{T}/2}B_v^\R{T}T_v^{-1}\Big)\R{row}_6 \\
&\begin{bmatrix}
I & 0 & 0 & 0 & (A_w)^\R{T} & 0 & 0\\
0 & I & 0 & 0 & 0 & (A_v)^\R{T} & 0\\
0 & 0 & S_w & 0 & -Q_w(A_w)^\R{T} & 0 & 0\\
0 & 0 & 0 & S_v & 0 & -Q_v(A_v)^\R{T} & 0 \\
0 & 0 & 0 & 0 & -T_w & 0 & B_wQ^{-1/2}G \\
0 & 0 & 0 & 0 & 0 & -T_v  & -B_vR^{-1/2}H \\
0 & 0 & 0 & 0 & 0 & 0 & \Omega
\end{bmatrix}
\end{aligned}
\end{equation*}
where
\begin{equation}
\label{OmegaMatrix}
\begin{aligned}
\Omega &= \Omega_G + \Omega_H =
G^\R{T}Q^{-\R{T}/2}B_w^\R{T}T_w^{-1}B_wQ^{-1/2}G 
+
H^\R{T}R^{-\R{T}/2}B_v^\R{T}T_v^{-1}B_vR^{-1/2}H.
\end{aligned}
\end{equation}
Note that $\Omega$ is symmetric positive definite. Note also that
$\Omega$ is block tridiagonal, since
\begin{enumerate}
\item $\Omega_H$ is block diagonal.
\item $Q^{-\R{T}/2}B_w^\R{T}T_w^{-1}B_wQ^{-1/2}$ is block
diagonal, and $G$ is block bidiagonal, hence $\Omega_G$
is block tridiagonal.
\end{enumerate}
Solving  system (\ref{NewtonSystem}) requires inverting the block
diagonal matrices $T_v$ and $T_w$ at each iteration of the damped
Newton's method, as well as solving an equation of the form $\Omega
\Delta x = \varrho$. The matrices $T_v$ and $T_w$ are block diagonal, 
with sizes $Nn$ and $Nm$, assuming $m$ measurements at each time point. 
Given that they are invertible (see~\eqref{NullProperty}), these inversions 
take $O(Nn^3)$ and $O(Nm^3)$ time. 
 Since $\Omega $ is block
tridiagonal, symmetric, and positive definite, $\Omega
\Delta x = \varrho$ can be solved in $O(Nn^3)$ time using the block
tridiagonal algorithm in \citep{Bell2000}. 
The remaining four back
solves required to solve (\ref{NewtonSystem}) can each be done in
$O(Nl)$ time, where we assume that $A_{v(k)}\in \mB{R}^{n \times l}$ 
and $A_{w(k)} \in \mB{R}^{m \times l}$
at each time point $k$.

\bibliographystyle{plain}
\bibliography{biblio}

\end{document}